\documentclass{amsart}
\usepackage{amssymb}
\usepackage{tikz}
\usetikzlibrary{shapes.geometric}
\usepackage{mdwlist}
\usepackage[all]{xy}
\usepackage[mathscr]{eucal}

\def\nb#1{}

\newtheorem{theorem}{THEOREM}
\newtheorem{lemma}[theorem]{LEMMA}

\newtheorem{remark}[theorem]{REMARK}
\newtheorem{definition}[theorem]{DEFINITION}

\newtheorem{examples}[theorem]{EXAMPLES}

\newcommand{\sfy}{\mathsf{y}}
\newcommand{\sfn}{\mathsf{n}}
\newcommand{\sfr}{\mathsf{r}}
\newcommand{\sfb}{\mathsf{b}}
\newcommand{\sfg}{\mathsf{g}}

\newcommand{\cv}[1]{#1\breve{\ }}
\newcommand{\id}{1'}

\newcommand{\cp}{\mathbin{\circ}}

\newcommand{\join}{+}
\newcommand{\comp}{\mathbin{;}}
\def\pw{\mathscr{P}}
\newcommand{\up}{\textup}
\newcommand{\Id}{\operatorname{Id}}

\def\de{\delta}
\def\ep{\varepsilon}

\def\QRA{{\sf QRA}}
\def\NA {{\sf NA}}
\def\Los{\L{o}\'{s}}

\def\i{{\sf i}}
\def\j{{\sf j}}

\def\g{{\sfg}}
\def\E{{ \mathscr{E}}}

\def\r{{\sfr}}
\def\y{{\sfy}}
\def\n{{\sfn}}
\def\p{{\sf p}}
\def\q{{\sf q}}

\def\x{{\times}}
\def\restr #1{{\restriction_{#1}}}

\def\Cm{\mathfrak{Cm}}

\def\RRA{{\sf RRA}}

\newcommand{\reals}{\mbox{\(\mathbb R\)}}

\def\set#1{{ \{ #1 \} }}
\def\c#1{{\mathcal #1}}
 \tikzset{vertex/.style={draw, shape=circle, fill=white, inner sep=0pt, minimum size=4pt},}
 
\title{Algebraic foundations for qualitative calculi and networks}
\author{Robin Hirsch}\author{Marcel Jackson}\author{Tomasz Kowalski}
 \thanks{The second author was supported by ARC Future Fellowship FT120100666
   and Discovery Project DP1094578.  The third author was supported by ARC
   Future Fellowship FT100100952.}
 
\begin{document}
\maketitle
 
\begin{abstract} 
Binary Constraint Problems have traditionally been considered as Network
Satisfaction Problems over some relation algebra.  A constraint network is
satisfiable if its nodes can be mapped into some representation of the relation
algebra in such a way that the constraints are preserved.  A qualitative
representation $\phi$ is like an ordinary representation, but instead of
requiring that $(a\comp b)^\phi$ is the composition $a^\phi\cp b^\phi$ of  the relations $a^\phi$ and  $b^\phi$, as we do for ordinary
representations, we only require that $c^\phi\supseteq a^\phi\cp b^\phi \iff c\geq
a\comp b$, for each $c$ in the algebra.  A constraint network is qualitatively
satisfiable if its nodes can be mapped to elements of a qualitative
representation, preserving the constraints.  If a constraint network is
satisfiable then it is clearly qualitatively satisfiable, but the converse can
fail, as we show.  However, for a wide range of relation algebras including the point
algebra, the Allen Interval Algebra, RCC8 and many others, a network  is
satisfiable if and only if it is qualitatively satisfiable. 

Unlike ordinary composition, the weak composition arising from qualitative representations  need not be associative, so we can
generalise by considering network satisfaction problems over non-associative
algebras.  We prove that computationally, qualitative representations have many
advantages over ordinary representations:  whereas many finite relation algebras
have only infinite representations, every finite qualitatively representable
algebra has a finite qualitative representation; the representability problem
for (the atom structures of) finite non-associative algebras is {\bf
  NP-complete}; the network satisfaction problem over a finite qualitatively
representable algebra is always in {\bf NP}; the validity of equations over
qualitative representations is {\bf co-NP-complete}.  On the
other hand we prove that there is no finite axiomatisation of the class of
qualitatively representable algebras. 

\end{abstract}
\nb{R2}
\section{Introduction}
Computer scientists have been solving systems of binary constraints for a long
time.  Temporal reasoning, for example, is often dealt with by solving a set
of temporal constraints between events, represented in 
a \emph{network}: a finite\nb{R3} complete graph
whose edges are labelled by a choice of alternative temporal relations. The
network is \emph{satisfiable} if it is possible to map the nodes to temporal events
in such a way that each pair of nodes is mapped to a pair of events satisfying one of the
alternative temporal relations labelling that edge.   An algebra of these relations, in one of the simplest
cases, is the \emph{point algebra}, where the primitive
alternative relations are $=, <, >$ and the events are points on a linear flow
of time. Relational compositions of
these basic relations are recorded in the table
in the upper left corner of Figure~\ref{fig}, where $\cp$ denotes  composition of binary relations.

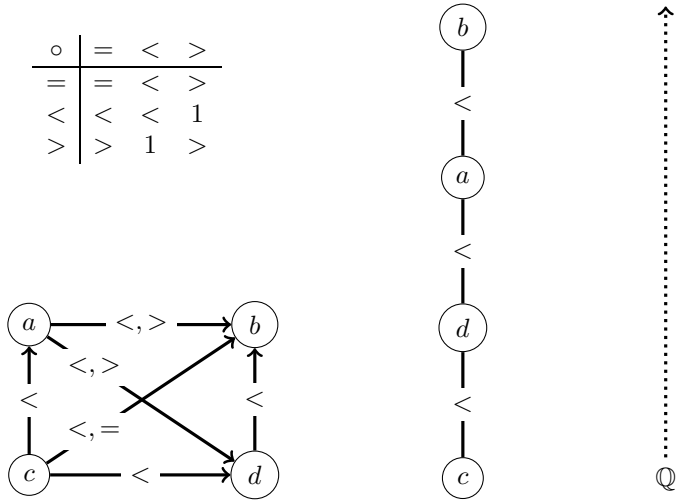
\begin{figure}
\begin{tikzpicture}
\node at (1.3,5) (v5) {\begin{tabular}{c|ccc}
$\cp$ & $=$ & $<$ & $>$ \\
\hline 
$=$ & $=$ & $<$ & $>$ \\
$<$ & $<$ & $<$ & $1$ \\
$>$ & $>$ & $1$ & $>$ \\
\end{tabular}};
\path (0,0) node[draw,shape=circle] (v0) {$c$} 
-- (0,2) node[draw,shape=circle] (v1) {$a$} 
-- (3,2) node[draw,shape=circle] (v2) {$b$} 
-- (3,0) node[draw,shape=circle] (v3) {$d$};
\draw[->,very thick] (v0) -- node[fill=white] {$<$} (v1);
\draw[->,very thick] (v1) -- node[fill=white] {$<,>$} (v2);
\draw[->,very thick] (v0) -- node[near start, fill=white] {$<,=$} (v2);
\draw[->,very thick] (v0) -- node[fill=white] {$<$} (v3);
\draw[->,very thick] (v3) -- node[fill=white] {$<$} (v2);
\draw[->,very thick] (v1) -- node[near start,fill=white] {$<,>$} (v3);
\end{tikzpicture}
\hskip2cm
\begin{tikzpicture}
\path (0,0) node[draw,shape=circle] (v0) {$c$} 
-- (0,2) node[draw,shape=circle] (v1) {$d$} 
-- (0,4) node[draw,shape=circle] (v2) {$a$} 
-- (0,6) node[draw,shape=circle] (v3) {$b$};
\draw[very thick] (v0) -- node[fill=white] {$<$} 
(v1) -- node[fill=white] {$<$} (v2) -- node[fill=white] {$<$} (v3);
\end{tikzpicture}
\hskip2cm
\begin{tikzpicture}
\draw[->,dotted, very thick] (0,0) node[anchor=north] {$\mathbb{Q}$} -- (0,6);
\end{tikzpicture}
\caption{The point algebra: composition table, a constraint network, a solution,
a strong representation.  Here $1$ denotes $\set{=, <, >}$ (all three relations are permitted).
\label{fig}}
\end{figure}

Consider the network over the point algebra given in  Figure~\ref{fig}.  
It is satisfiable in a linear order of just four distinct points, but a
\emph{representation} of the point algebra has to be infinite, because ${<} \mathbin{\cp} {<}$ is identical to $<$, 
which entails not
just transitivity (${<} \mathbin{\cp} {<}$ is contained in $<$) but also density ($<$ is
contained in ${<} \mathbin{\cp} {<}$).  This discrepancy between an infinitely
\emph{representable} algebra and finitely \emph{satisfiable} 
networks over it, is not too serious in this case because every
finite linear order embeds into the rational numbers. Hence, if a  network has a
solution in some linear order then it can be embedded into a representation of the
point algebra. 

To deal with temporal intervals rather than point-events, 
the Allen Interval Algebra \cite{All83} is very commonly used. 
Here, we have thirteen alternative primitive relations between
intervals on a linear flow of time. For the Allen Interval Algebra, a solution to a 
constraint network would be a finite set of intervals in a linear order with an
appropriate relation holding between each pair, but a 
representation of this algebra is again infinite: 
it consists of ordered pairs taken from a dense
linear order without endpoints \cite{LaMa94}.    
And again, there is no real discrepancy here because every finite arrangement of
intervals in a  linear order embeds into a set of intervals of rational numbers.

But when one tries to generalise the above examples to apply relational
reasoning in other domains, the discrepancy becomes a real issue.
A very clear example of this occurs in \emph{spatial reasoning}, where an
analogue of the Allen Interval Algebra with relations between spatial regions is
used. This algebra is called RCC8. One of the basic
relations considered in RCC8 is \emph{external connectedness}  
($EC$), whose intended interpretation is that $x{EC}y$ if regions $x$ and $y$ 
touch at the borders but only at the borders, for example as in the left-hand side of 
Figure~\ref{regions}.
Now, RCC8 requires that $EC\cp EC \supseteq EC$, which is reasonable if we think
of regions topologically as open balls (open disks in $\mathbb{R}^2$; more
generally, open sets with boundaries of genus 0), 
because then for any $xECy$ we can find a $z$
with $xECzECy$, as in the left-hand side of Figure~\ref{regions}. 
However, in real-life applications, this assumption is not always warranted.
For it happens that one region can be completely surrounded by another. This is
the spatial relation San Marino bears to Italy,  the Vatican City to Rome, and
Lesotho to South Africa. When this happens, as in the right-hand side
of Figure~\ref{regions}, where region $y$ is the annulus surrounding region $x$,  
we have $xECy$, but $(x,y)\notin EC\cp EC$.
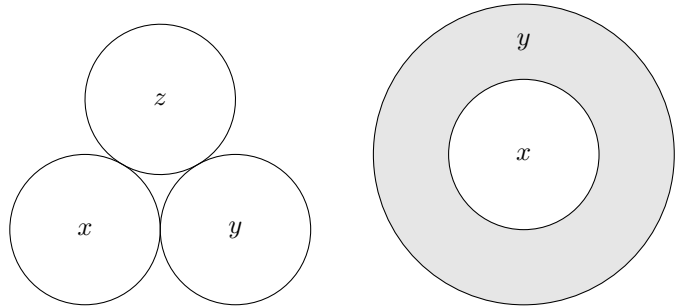
\begin{figure}
\begin{tikzpicture}
\draw (0,0) node {$x$} circle (1cm);
\draw (2,0) node {$y$} circle (1cm);
\draw[rotate=60] (2,0) node {$z$} circle (1cm);
\end{tikzpicture}
\qquad
\begin{tikzpicture}
\draw[fill=gray!20] (0,0) circle (2cm);
\draw[fill=white] (0,0) node {$x$} circle (1cm);
\node at (0,1.5) {$y$};
\end{tikzpicture}
\caption{\label{regions}Some relations between spatial regions}
\end{figure}

The problem was identified, and rightly diagnosed to be an anomaly. The remedy
was to consider an algebra of binary relations  where  relational composition is
replaced by another binary operation, called 
\emph{weak composition}, approximating real composition from above 
(see, for example~\cite{LR04}).   Weak composition is defined in such a way that
when $(x, y)$ is in  the weak composition of $EC$ with $EC$ it is not
\emph{mandatory} that there is a $z$ with $xECzECy$, it is merely
\emph{permitted} that such a $z$ should exist.  (Note, below we distinguish two different meanings of weak composition from the literature: either it is merely permitted that such a $z$ should exist, or it is permitted and additionally there must be $x', y', z'$ such that $x'ECz'ECy'$ i.e. the composition has to be realised at least once. The two corresponding types of representations we consider are \emph{feeble} and
\emph{qualitative} respectively, see below.)
Since then, an impressive body
of research has been conducted in qualitative reasoning based on these notions of
weak composition (see \cite{KER:9483409} for a survey).  

One restriction to this framework of qualitative reasoning is  that the
identity relation is assumed to be an indivisible primitive relation (an
\emph{atom} of the algebra), so that unary properties of states  cannot be
expressed directly.   So, for example, suppose we want to assert that a certain
time interval occurs \emph{during} an interval where the printer is working.
In the Allen Interval Algebra we can express that one interval $\i$ occurs
\emph{during} another interval $\j$ but we cannot assert properties of the
interval $\j$.  In our framework we may introduce a subidentity atom $\mathsf{w}$ with
intended semantics $\j\mathsf{w} \j$ if and only if the printer is working on the
interval $\j$.  A similar extension was required for \emph{Kleene Algebra},
where it was necessary to introduce the \emph{test} operator in order to
express properties of states, rather than relations \cite{Koz97}. Here, we make no
assumption that the identity is atomic.

A further restriction  is that the definition of weak composition only applies in the 
setting of finitely many primitive relations.    
However, there are applications in Artificial Intelligence where infinitely many different relations are needed.  A number of different researchers who wished 
to add quantitative reasoning to qualitative constraint systems, adopted languages containing infinitely many constraints (e.g. \cite{KL91,Mei96,Hir96} or see \cite{OW08} for a survey of metric temporal logic, also see Example~\ref{examples}.\ref{ex:final} below).\nb{R1}

In this article we define the weak composition of two binary relations in general;
our definition coincides with the original one where it applies, but covers a wider range of algebras of binary relations.
We also define a corresponding notion of qualitative representation.  In a
classical representation of a relation algebra (referred to henceforth as a \emph{strong representation}), given two points $x, y$ for which
it is \emph{consistent} for there to be a $z$ with $(x, z)\in a,\; (z,
y)\in b$ it is then \emph{mandatory} that such a point $z$ exists.  This requirement is
relaxed in a qualitative representation, see Definition~\ref{qualrep} below.  A constraint network is qualitatively
satisfiable if 
it embeds into a qualitative representation.  We will see that this corresponds much
more closely to the intuitive approach to binary constraint
problems, such as those illustrated in Figure~\ref{fig}. There, the
four element chain is 
in fact a qualitative representation of the point algebra.

\subsection {Fixed vs. arbitrary representations} \label{sec:1.1} When considering the constraint satisfaction problem in its general setting one 
typically has a fixed domain for each variable and the interpretation of any relation symbols used in the constraints is also fixed. \nb{R1, R3}  
In keeping with that, a great deal of the research into qualitative reasoning focusses on a fixed set of relations on a fixed base set and considers
the satisfiability of constraints in that setting, in other words the issue is
the satisfiability of constraints in a single, fixed representation.  To
illustrate the value of this approach, suppose we wish to schedule a series of
meetings in a discrete flow of time where there are exactly four time points, as
in Figure~\ref{fig}.  A set of binary constraints on the scheduling of events
could be represented as a network over the point algebra, but the question to
consider is not whether the network is satisfiable in \emph{some} qualitative
representation, it is whether the network can be satisfied in the qualitative
representation of the point algebra consisting of a linear order of four
points.  For this kind of problem it is the satisfiability of a network in a
fixed representation that should be considered.  (The complexity of this problem
can be fairly high, for example in  \cite{Lee14} a single model based on
\emph{dipoles} in the real plane, (i.e. elements of  $\reals^4$) is adopted and
it is shown that the network satisfaction problem is  $\exists\reals$-complete
for various algebras of relative directional constraints over this model.  For a
more extreme case, consider a graph algebra with three primitive constraints:
equals, adjacent and non-adjacent.  Let $S$ be an undecidable set of finite
connected graphs  and let $G$ be the disjoint union of all graphs in $S$. Given
any finite connected graph $F$ we can define a network on the same set of nodes,
edges are labelled `adjacent' and irreflexive non-edges are labelled
`non-adjacent'.  Since this network is qualitatively satisfiable in $G$ if and only if $F\in
S$, the problem of determining whether a network is satisfiable in $G$ is
undecidable.)

 However, there are applications where the representation is not fixed, for example in spatial reasoning using RCC8 the exact topography of the relations between regions may not be known.  Many of the theorems proved in the  previously cited papers establish properties that hold over a whole class of representations, e.g. results in \cite{Ren02} refer to the solvability of RCC8 constraints over non-empty, regular, closed regions, regardless of the particular topology under consideration.    The main decision problem considered in the current paper is to decide if a given network is satisfiable in an arbitrary qualitative representation.

Results\nb{R1, R2} obtained in this paper indicate that qualitative
representations have computational advantages over strong representations.  All
the algebras of relations mentioned above have strong representations, but only
on infinite base sets.  In contrast, we show that if a finite algebra of
relations (formally, a non-associative algebra) has a qualitative representation
then it has one on a finite base set.  A consequence is that the problem of
determining whether a finite non-associative algebra is  qualitatively
representable is in {\bf NP} (indeed it is {\bf NP-complete}) 
whereas the strong representation problem is known to be undecidable \cite{HH7}.

Furthermore, although it happens to be the case that a consistent, atomic network of constraints (a consistent network with only a single primitive relation on each edge) is always satisfiable for the algebras of relations previously mentioned, and hence a polynomial time, non-deterministic algorithm can solve the network satisfaction problem for these algebras by guessing a primitive label for each edge and then checking their consistency,  this does not work for other algebras of spatial relations.  For example,  consider\nb{R3} the \emph{interval and duration} algebra INDU whose twenty-five  primitive relations are  similar to Allen's thirteen primitive interval relations, but also determine whether the duration of the first interval is smaller, equal or greater in duration than the second interval \cite{PKS99}.  There are known, consistent, atomic INDU-networks which cannot be satisfied by intervals \cite[Figure~8.11]{Lig11}.  Thus consistency of a network does not suffice to prove that the network is satisfiable, even if the network is atomic.  Although the network satisfaction problem remains in {\bf NP} for INDU-networks, there are known relation algebras where the problem has much worse complexity \cite{Hir:undec}. 
 On the other hand, for any finite algebra of relations the network satisfaction problem over qualitative representations always belongs to {\bf NP}.  

Similarly, although the validity of equations valid over strong representations was shown to be undecidable by Tarski, the validity of equations valid over all qualitative representations is decidable (indeed it is {\bf co-NP-complete}).  Our conclusion is that qualitative representations are not only more appropriate to express the kind of contraints that arise from many applications, but they are more amenable to algorithmic reasoning.

\subsection{Historical remarks}
The structures we called algebras above, 
were conceived as \emph{calculi}: formal rules for manipulating
relations, invented and developed \emph{ad hoc}, to suit the
purpose at hand. This is evident in the naming: for example, RCC8 is so called
because it was originally developed in~\cite{RCC8-92} as 
\emph{Region Connection Calculus}, with 8 basic relations, hence the 
acronym RCC8 (although the names of the three authors might also have something
to do with it). Later, mathematicians observed that such calculi, including 
the point algebra, Allen Interval  Algebra and
RCC8, were examples of Tarski's Relation Algebras. As far as we
know, this observation was first made in~\cite{LaMa88,LaMa94}.  In this setting, the basic relations are boolean
atoms in a relation algebra,  the edges of a network are labelled by arbitrary
elements of the relation algebra and the network is satisfiable if its nodes can
be mapped into some representation of the relation algebra in such a way that
the label of an edge of the network holds at the corresponding two points in the
representation.    However, as we outlined above, there were difficulties in
restricting to strong representations, particularly for relation algebras such
as RCC8, and this led to the weaker notion of qualitative representation, now
very widely studied in knowledge representation and its applications, see
\cite{DWM01,LR04,LW06,MSW06,KER:9483409,RL05}, for example.

\subsection{Notation}
We deal with abstract algebras and concrete representations as binary relations and separate the notation, to some extent.   Working abstractly we use $+, -$ as the basic boolean operators and introduce standard abbreviations, such as $x\cdot y = -(-x+-y)$ and $x\leq y\leftrightarrow x+y=y$.  
The identity constant is $1'$, the converse operator is $\cv{}$ and any algebraic multiplication-like operator (including weak composition, below) will be denoted by $;$.  Working with concrete binary relations we write $\cup, \setminus$ for the operators corresponding to $+, -$ and we write $\Id_D=\set{(x,x):x\in D}$ for the identity relation over a domain $D$, corresponding to the abstract $1'$, though we may drop the subscript $D$ if it is clear from the context.  
The converse of a binary relation $r$ will be written as $\cv{r}=\set{(y, x):
  (x, y)\in r}$.  We write $r\cp s =\set{(x, y)\in D\times D: \exists z\in D\;
  (x, z)\in r\wedge (z, y)\in s}$ for the composition of two binary relations
$r, s$.  Our convention is that converse has highest precedence, followed by
composition which takes precedence over other operators, for example
$a\cdot b\comp \cv{c}$ denotes $a\cdot(b\comp (\cv{c}))$.\nb{R3}  For any set $S$ we write $\wp(S)$ for the power set of $S$.\nb{R2}

\section{Background}\label{sec:bckgrnd}

A \emph{qualitative calculus} has traditionally been 
 defined (see, for example,~\cite{LR04}) by specifying a finite 
partition $\Pi = (R_0,\dots,R_n)$
of the set $D\times D$, for some fixed (usually infinite) 
domain $D$, with the following properties:
\begin{enumerate}
\item The identity relation $\Id_D$ is an element of the partition, 
\item $\Pi$ is closed under relational converses, that is, $\cv{R}\in \Pi$
  for every $R\in \Pi$. 
\end{enumerate}
The set $\{R_0,\dots,R_n\}$ generates a boolean subalgebra $\mathcal{B}$ of 
$\pw(D\times D)$ under the usual set-theoretical operations. Clearly,
$R_0,\dots,R_n$ are atoms of $\mathcal{B}$; these include the identity
relation. Moreover, $\mathcal{B}$ is closed under relational converses.
However,~$\mathcal{B}$ is not in general closed under relational composition, 
as the example of Figure~\ref{regions} indicates. This is remedied by
considering \emph{weak composition} instead: an operation defined by 
$$
S\comp T = \bigcup\{R\in \Pi\colon R\cap (S\cp T)\neq \varnothing\}
$$ 
where $S\cp T$ stands for the true composition of $S$ and $T$. So defined,
$S\comp T$ is the smallest element of $\mathcal{B}$ containing $S\cp T$.

Thus, a qualitative calculus carries a natural algebraic structure of the type
of a relation algebra.  Viewed from an
abstract algebraic perspective, a qualitative calculus is a hybrid
object: an abstract algebra together with a concrete interpretation, or
a \emph{representation}. As in \cite{WestphalHW14}, one of our aims in this article is to  
separate the two sides of a qualitative calculus, into \emph{syntax} (algebra)
and \emph{semantics} (representation), and so investigate the foundations of
qualitative calculi in a manner similar to  
model theoretical analysis of classical mathematics.

Before we state the basic definitions, let us  recall the two generalisations
that we adopt from the outset. Firstly, we will lift the finiteness assumption.
It is not necessary for a definition of weak composition, and from a universal algebraic
point of view admitting infinite algebras is more natural, furthermore, as we have seen, there are applications where it is desirable to include infinite relation algebras, for example when we wish to express \emph{metric constraints}.\nb{R1, R2} Secondly, we do not require that the identity is an atom, as subidentity relations provide a natural way of modelling
\emph{properties}, that is, subsets of the domain, by representing
 a set $Z\subseteq D$ by the relation $\{(z,z)\colon z\in Z\}$.   

\begin{definition}\label{def:parcel}
Let $D$ be a set and let $\c S$ be a set of binary relations over $D$, that is, $\c
S\subseteq\pw(D\times D)$.  $\c S$ is a \emph{herd} if \nb{R1}
\begin{enumerate}
\item \label{en:top} $\c S$ forms a boolean set algebra with top element $D\times D$, so $\c S$
  is closed under finite intersections and complement relative to $D\times D$,
\item\label{en:id} $\Id_D\in \c S$,
\item\label{en:conv}  If $A\in \c S$ then the converse relation $\cv{A}$ is in $\c S$.
\end{enumerate}
\end{definition}
 In a herd $\c S$ given any two elements $A, B\in \c S$ if there is a minimal $C\in\c S$ containing $A\circ B$ then we say that the \emph{weak composition} of $A$ and $B$ is $C$.  If $\c S$ is finite, then such a minimal element is sure to exist, since $\c S$ is closed under finite intersections.

For herds with infinitely many relations, the weak composition of two elements is not always defined (a
minimal element containing $A\circ B$ may not exist), however the case we are
interested in is the case where the weak composition of $A$ and $B$ is defined,
and for the abstract algebraic structure corresponding to herds we include a
binary composition operator $;$, so the signature is the same as  that of a
relation algebra: it is a boolean algebra with an extra nullary operation
$\id$ for identity, a
unary operation $\cv{}$ for converse, and a binary operation $\comp$ to denote
weak composition and it obeys all the axioms defining a relation algebra except
perhaps associativity (see below). Maddux calls such an algebra a
\emph{non-associative relation algebra} \cite[Definition~1.2]{Ma82}, or
non-associative algebra for short.  \cite{LR04} already observed that  their
qualitative calculi are non-associative algebras and it is easily verified that
the herds  considered here are non-associative algebras too, in the cases where
weak composition is defined. 

An algebra $\c A=(A, 0, 1, +, -, 1', \cv{ }, \comp)$ of the type
of relation algebras belongs to the variety   
  $\mathsf{NA}$ of \emph{non-associative algebras}, if 
\begin{enumerate} 
\item $(A, 0, 1, +, -)$ is a boolean algebra, 
\item $(A, 1', \cv{ }, \comp)$ is an involuted monoid, i.e. it satisfies
\begin{enumerate}
\item $1'\comp x = x = x\comp 1'$
\item $\cv{\cv{x}} = x$
\item $\cv{(x\comp y)} = \cv{y}\comp\cv{x}$
\end{enumerate}
\item $\cv{ }$ and $\comp$ are normal additive operators, that is
\begin{enumerate}
\item $\cv 0=  x\comp 0 =0$,
\item $\cv{(x+y)} = \cv{x}+\cv{y},\; x\comp(y+z) = (x\comp y)+(x\comp z)$
\end{enumerate}
\item $x\comp y\cdot\cv{z}=0$ if and only if $y\comp z\cdot \cv{x}=0$ (Peircean law)
\end{enumerate}
  By additivity, the operators are monotone, e.g. $y\leq z\rightarrow x;y\leq
  x;z$, etc. Since the operators $\cv{}, ;$ are \emph{conjugated} it turns out
  that every non-associative algebra is \emph{completely additive}, i.e. if $S$
  is a subset of the elements of $\c A$ with a supremum $\Sigma S$ then
  $\cv{(\Sigma S)}$ is the supremum of $\set{\cv{s}:s\in S}$ and for any $a\in\c
  A$ the element $a;\Sigma S$ is the supremum of $\set{a;s:s\in S}$
  \cite[Theorem 1.14]{JT51}. 

In the following\nb{R3}  we define a qualitative representation as an
\emph{isomorphism} from an non-associative algebra to a herd.
In \cite[Definition~3]{LR04} the corresponding definition is only required to be
a homomorphism, injectivity is not required.   One difficulty with this weaker notion of representation is that trivial map from an arbitrary non-associative algebra to the herd on an empty domain is a non-injective homomorphism, and we  wish to exclude this trivial
case. 
\begin{definition}\label{qualrep}
Let $\c A=(A, 0, 1, +, -, 1', \cv{ }, \comp )$ be a  non-associative algebra.
A \emph{qualitative representation} $\phi$ of an algebra $\c A$ is an injection
to a herd $\c S$ of binary relations over base $D$, such that 
\begin{enumerate}
\item\label{en:total} $0^\phi=\varnothing, \; 1^\phi=D\times D,\; (1')^\phi={\Id}_D$,
\item $(a+b)^\phi=a^\phi\cup b^\phi,\; (-a)^\phi=(D\times D)\setminus a^\phi$,
\item $(\cv{a})^\phi=\cv{(a^\phi)}$,
\item\label{en:abc} $c^\phi\supseteq a^\phi \cp  b^\phi \leftrightarrow c\geq a\comp b$
\end{enumerate}
for all $a, b, c\in A$.  If ${\c A}$ has a qualitative representation, then we
say that ${\c A}$ is a \emph{qualitatively representable algebra}.
The class of all qualitatively representable algebras we will denote by \QRA.

If $(a\comp b)^\phi = a^\phi\cp b^\phi$ for all $a, b\in \c A$ then the  qualitative representation $\phi$ is a
\emph{strong  representation}.    \RRA\ is the variety generated by the class of all strongly representable relation algebras.
\end{definition}  
If $\phi$ is a qualitative representation of a non-associative algebra $\c A$, observe that $a;b$ is always defined (for $a, b, \in\c A$) and  Definition~\ref{qualrep}.\ref{en:abc} requires that it is the minimal solution of $c\in\c A,\;c^\phi\supseteq a^\phi\circ b^\phi$.

 The class of strongly representable relation algebras is already known to be extremely complicated: without finite
 axiomatisation in first order logic, with undecidable equational theory and
 with undecidable membership problem for finite algebras~\cite{HH7}.  The class
 \QRA\ is known to be a proper subclass of \NA\ \cite{LR04,WestphalHW14}, or see
 Example~\ref{examples}.\ref{ex:ngqr} below.  We show below that the class \QRA\
 has intermediate difficulty: it is also without a finite axiomatisation
 (Theorem \ref{thm:nfa}) but it is {\bf NP-complete} to decide membership for finite algebras (Theorem \ref{thm:NP}).

The definition of qualitative representation (Definition \ref{qualrep}) is based on the definition of weak composition for partition
schemes given in \cite[\S2.3]{LR04}, however our definition applies not
just to finite partition schemes, it works even for infinite herds.  Moreover,  Ligozat and Renz appear to include two distinct notions of weak composition.  In \cite[abstract]{RL05}\nb{R1} they require that weak composition is  ``the strongest relation
containing the real composition'', in agreement with much of the relevant literature and in agreement with Definition \ref{qualrep} above.\nb{R1, R2, R3}  However, in 
\cite[Definition~3]{LR04} and in \cite[Definition~2]{Ligozat2005}  only the right to left implication of Definition~\ref{qualrep}.\ref{en:abc} is required, that is, they only require $(a\comp b)^\phi\supseteq a^\phi\cp b^\phi$ and do not insist that $c=a;b$ is the minimal solution of $c^\phi\supseteq a^\phi\cp b^\phi$ as $c$ ranges over elements of the algebra.  We call this looser definition of a qualitative representation a \emph{feeble} representation and investigate it separately in Section~\ref{sec:feeble}.
Example~\ref{examples}.\ref{2-ex}  and \ref{examples}.\ref{ex:ngqr} below show that there is a real discrepancy between qualitative representations and feeble representations. (\cite[Proposition~2]{WestphalHW14} dispute a quite separate point in \cite{RL05}, concerning the notion of ``closure under constraints''.)

\begin{lemma}\label{lem:xyz}
Let $\phi:\c A\rightarrow \wp(D\times D)$ respect the boolean operators, the identity and the converse operator, i.e. $1^\phi=D\times D,\; (-a)^\phi = (D\times D)\setminus a^\phi,\; (a+b)^\phi=a^\phi\cup b^\phi,\; (1')^\phi=\Id_D$ and $(\cv{a})^\phi=\set{(y, x):(x, y)\in a^\phi}$.  The following are equivalent
\begin{enumerate}
\item $\phi$ is a  qualitative representation \label{en:gqr}
\item  for all $a, b, c\in\c A$\label{en:abcA}
\[(a;b\cdot c\neq 0)\;\leftrightarrow \;\exists x, y, z\in D ((x, y)\in a^\phi\wedge(y, z)\in b^\phi\wedge(x, z)\in c^\phi).\] 
\label{en:xyz}
\end{enumerate} 
\end{lemma}
\begin{proof}
For \eqref{en:gqr} $\Rightarrow$ \eqref{en:xyz}, suppose $\phi$ is a  qualitative representation.  If there are $x, y, z\in D$ such that $(x, y)\in a^\phi,\;(y, z)\in b^\phi,\; (x, z)\in c^\phi$ then $(a;b\cdot c)^\phi=(a;b)^\phi\cap c^\phi\supseteq a^\phi\cp b^\phi \cap c^\phi\neq \varnothing$, so $a;b\cdot c\neq 0$.  If no such $x, y, z$ exist then $a^\phi\cp b^\phi\subseteq (-c)^\phi$ and since $(a;b)^\phi$ is smallest such that $(a;b)^\phi\supseteq a^\phi\cp b^\phi$ it follows that $(a;b)^\phi\subseteq (-c)^\phi$, \/ $(a;b)^\phi\cap c^\phi=\varnothing$, so $a;b\cdot c=0$. 
 
 Conversely assume \eqref{en:xyz}.   If $(a;b)^\phi\not\supseteq a^\phi\cp b^\phi$ then there are $x, y, z$ with $(x, y)\in a^\phi,\; (y, z)\in b^\phi$ and $(x, z)\in (-(a;b))^\phi$, hence by \eqref{en:xyz} $a;b\cdot (-(a;b))\neq 0$, a contradiction.  We conclude $(a;b)^\phi\supseteq a^\phi\cp b^\phi$.  If $c^\phi\supseteq a^\phi\cp b^\phi$ then there do\nb{R3} not exist $x, y, z$ such that $(x, y)\in a^\phi,\; (y, z)\in b^\phi$ and $(x, z)\in(-c)^\phi$ so by \eqref{en:xyz}, $a;b\cdot(-c)=0$ and $c\geq a;b$, thus $c=a;b$ is the minimal solution in $\c A$ of $c^\phi\supseteq  a^\phi\cp b^\phi$, this proves that $;$ is correctly represented as weak composition.
  \end{proof}

\section{Atom Structures and Examples}

In the case of an atomic algebra, a convenient way of specifying the operators
is by defining its \emph{atom structure}. 
\begin{definition}
Let $X$ be the set of atoms \up(minimal, non-zero elements\up) of a
non-associative algebra $\c A$.  The atom structure $At(\c A)$ is defined as 
\[At(\c A)=(X, E, \cv{ }, C)\]
where $E$ is the set of atoms below the identity, $\cv{ }$ is the converse
function restricted to atoms, and $C$ is the set of consistent triples of atoms,
that is, those triples of atoms $(a, b, c)$ such that $a\comp b\geq c$.  

Conversely, given $(X, E, \cv{ }, C)$, where $E\subseteq X$, $\cv{ }$ is a
unary function on $X$ and $C\subseteq X^3$ we may define the \emph{complex
  algebra} $\Cm(X, E, \cv{ }, C) = (\pw(X), \varnothing, X, \cup, \setminus, E,
\cv{ }, \comp )$ where $\pw(X)$ is the power set of $X$, \/ $E$ is the identity
element, $\cv{ }$ is extended to sets of atoms by taking unions, and
multiplication is defined by $S\comp T=
\{x\in X\colon \exists s\in S,\; t\in T\; (s, t,  x)\in X\}$. 
\end{definition}
For any atomic non-associative algebra $\c A$ with atoms $X$, the map defined by $a\mapsto\set{x\in X:x\leq a}$ is an embedding of  $\c A$ into $\Cm(At(\c A))$ and in the case where $\c A$ is complete and atomic, this map is surjective, i.e. an isomorphism.\nb{R3}
Observe, for finite algebras, that the number of atoms is the logarithm (base
two) of the number of elements of the algebra.  It is clear, by additivity, that the constant $1'$ and the operators $\cv{}, ;$ are determined by the atom structure, when the atom structure is finite and in fact, as we noted earlier, every non-associative algebra is \emph{completely} additive so the operators of an arbitrary atomic, non-associative algebra are determined by its atom structure.  

The next lemma is proved in \cite[Theorem 2.2(2)]{Ma82}.\nb{R1}
\begin{lemma}\label{lem:na}
Let $(S, E, \cv{ }, C)$ consist of a set $S$, a subset $E\subseteq S$, a unary
function $\cv{ }\colon S\rightarrow S$ satisfying $\cv{\cv{s}}=s$, and subset
$C\subseteq S\times S\times S$.  The following are equivalent 
\begin{itemize}
\item $(S, E, \cv{ }, C)$ is the atom structure of some non-associative algebra,
\item  For all $a, b, c\in S$ we have $a=b$ iff there is $e\in E$ such that 
$(e, a, b)\in C$, and  if $(a, b, c)\in C$ then 
$(\cv{b}, \cv{a}, \cv{c})\in C$ and $(\cv{c}, a,\cv{b})\in C$. 
\end{itemize}
\end{lemma}

The six triples $(a, b, c), (b, \cv{c}, \cv{a}), (c, \cv{b}, a), (\cv{a}, c, b),
(\cv{b}, \cv{a}, \cv{c}), (\cv{c}, a, \cv{b})$ are called the \emph{Peircean
  transforms} of $(a, b, c)$. In practice, the triples are given in
\emph{composition tables} such as the one used in Figure~\ref{fig} to define the
point algebra. The entry for $a\comp b$ is the join of the 
set $\{c\colon (a,b,c)\in C\}$, so if $(a,b,c_1),(a,b,c_2),(a,b,c_3)\in C$, the
entry for $a\comp b$ will be $c_1+c_2+c_3$.

In the following examples we define some finite non-associative algebras by
giving their atom structures, except for the final example which is defined directly. 
\begin{examples}\label{examples}
\begin{enumerate}
\item
\label{2-ex}
The first of our  non-associative algebras has three atoms, $\{e, e', a\}$, hence
eight elements.  The identity is $e+e'$, each element is self-converse,
multiplication is defined in the table on the left below
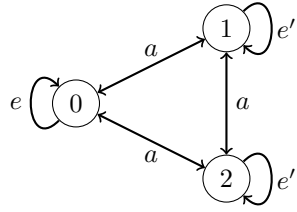
\[ 
\begin{tikzpicture}
\node at (-5,0){
$\begin{array}[t]{c|ccc}
\comp &e&e'&a\\
\hline
e&e&0&a\\
e'&0&e'&a\\
a &a&a&1
\end{array}$};
\tikzset{vertex/.style = {shape=circle,draw,minimum size=1.5em}}
\tikzset{edge/.style = {->}}
\node[vertex] (x) at (0,0) {$0$};
\node[vertex] (y) at (2,1) {$1$};
\node[vertex] (z) at (2,-1) {$2$};
\draw[thick,<->](x) to (y);
\draw[thick,<->](x) to (z);
\draw[thick,<->](z) to (y);

\draw [thick,->] (y) to [thick,out=45,in=315,looseness=4] (y);
\draw [thick,->] (x) to [thick,out=225,in=135,looseness=4] (x);

\draw [thick,->] (z) to [thick,out=45,in=315,looseness=4] (z);
\node at  (-.8,0)  {$e$};
\node at  (2.8,1)  {$e'$};
\node at  (2.8,-1)  {$e'$};

\node[above] at (1,.5) {$a$};
\node[below] at (1,-.5) {$a$};
\node[right] at (2,0) {$a$};

  \end{tikzpicture}
\]
Multiplication is not associative, for example $(e\comp e')\comp a = 0\comp a=0$ but $e\comp (e'\comp a)=e\comp a=a$. 
This non-associative algebra has a qualitative representation $\phi$ over a base of three points $\set{0,1,2}$, shown on the right, above:
$e^\phi=\Id_{\set{0}},\;  (e')^\phi=\Id_{\set{1,2}}$ and $a^\phi=\set{(x, y):x\neq y<3}$,  $\phi$ is defined on sums of atoms by
additivity.

Now let $\theta$ be obtained by restricting the qualitative representation $\phi$ to the base $\set{0, 1}$, illustrated below.
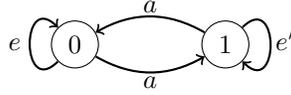
\[
\begin{tikzpicture}
\tikzset{vertex/.style = {shape=circle,draw,minimum size=1.5em}}
\tikzset{edge/.style = {->}}
\node[vertex] (x) at (0,0) {$0$};
\node[vertex] (y) at (2,0) {$1$};
\draw [edge,thick] (y) to[bend right] (x);
\draw [edge,thick] (x) to[bend right] (y);
\node [above] at (1,0.3) {$a$};
\node [below] at (1,-0.3) {$a$};
\node  at (-0.8,0) {$e$};
\node at (2.8,0.06) {$e'$};
\draw [thick,->] (y) to [thick,out=45,in=315,looseness=4] (y);
\draw [thick,->] (x) to [thick,out=225,in=135,looseness=4] (x);
  \end{tikzpicture}
  \]
Since $(b;c)^\theta\supseteq b^\theta\cp c^\theta$ for any $b, c$ in the algebra and all atoms are witnessed, $\theta$ is a feeble representation over the base $\set{0,1}$.  However, $a^\theta\cp a^\theta=\Id_{\set{0,1}}$ and $a;a=1$ is not minimal subject to containing $a^\theta\cp a^\theta$ \up(the minimal solution is $e+e'<1$\up), so $\theta$ is not a qualitative representation.

\item \label{3-ex}
This non-associative algebra has atoms $\{1', a, a', \times\}$ \up(so
sixteen elements\up).  All elements are self-converse, this time the identity $1'$
is an  atom, multiplication is defined by 
\[\begin{array}{c|cccc}
\comp &1'&a&a'&\times\\
\hline
1'&1'&a&a'&\times\\
a&a&1'+a&0&\times\\
a'&a'&0&1'+a'&\times\\
\times&\times&\times&\times&1'+a+a'
\end{array}
\]
Again, associativity fails because $(a\comp a')\comp a'=0\comp a'=0$ but
$a\comp (a'\comp a')=a\comp (1'+a')=a$.   A qualitative representation $\theta$ on base $\set{0,1,2}\cup\set{0', 1', 2'}$,  shown below, is defined by
$(1')^\theta=\Id_{\set{x, x':x<3}}$,\/ $a^\theta=\{(m, n)\colon m\neq n,\; m, n<3\}$,
$(a')^\theta = \{(m', n')\colon m\neq n,\;  m, n<3\}$, 
$\times^\theta = \{(m, n'), (m', n)\colon m, n<3\}$. Here and below, reflexive identity loops have been omitted\nb{R2}.
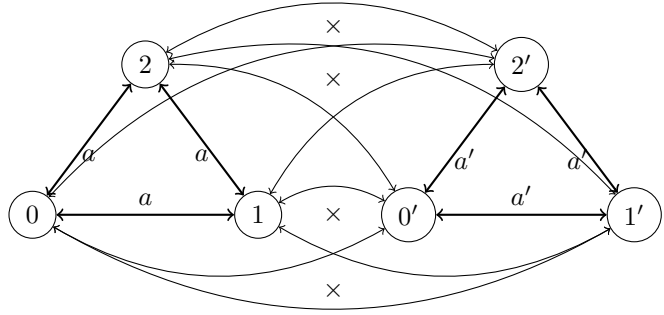
\[
\begin{tikzpicture}
\tikzset{vertex/.style = {shape=circle,draw,minimum size=1.5em}}
\tikzset{edge/.style = {<->}}
\node[vertex](x) at (0,0){$0$};
\node[vertex](y) at (3,0){$1$};
\node[vertex](z) at (1.5,2){$2$};
\draw[edge,thick](x) to (y);
\draw[edge,thick](y) to (z);
\draw[edge,thick](x) to (z);
\node [above] at(1.5,0) {$a$};
\node [below] at(.75,1) {$a$};
\node [below] at(2.25,1) {$a$};

\node[vertex](u) at (5,0){$0'$};
\node[vertex](v) at (8,0){$1'$};
\node[vertex](w) at (6.5,2){$2'$};
\draw[edge,thick](u) to (v);
\draw[edge,thick](v) to (w);
\draw[edge,thick](u) to (w);
\node [above] at(6.5,0) {$a'$};
\node [below] at(5.75,1) {$a'$};
\node [below] at(7.25,1) {$a'$};

\draw[edge,thin] (x) to [bend right] (u);
\draw[edge,thin] (x) to [bend right] (v);
\draw[edge,thin] (y) to [bend left] (u);
\draw[edge,thin] (y) to [bend right] (v);
\draw[edge,thin] (z) to [bend left] (u);
\draw[edge,thin] (z) to [bend left] (v);
\draw[edge,thin] (z) to [bend left] (w);
\draw[edge,thin] (x) to [bend left] (w);

\draw[edge,thin] (y) to [bend left] (w);

\node at (4, 2.5) {$\times$};
\node at (4, 1.8) {$\times$};
\node at (4, 0) {$\times$};
\node at (4, -1) {$\times$};
\end{tikzpicture}
\]

\item \label{ex:4}
The next  algebra has atoms  $\set{1', a, b,
  c}$ where the identity is $1'$, all atoms are self-converse and multiplication
is defined by the table below, known to Maddux as relation algebra $25_{65}$,
\cite{Mx:book}.  A strong representation of it is illustrated on the right. 
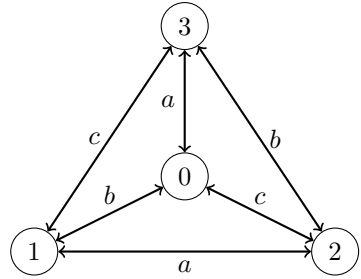
\[\begin{tikzpicture}
\node at (-4,2)
{$\begin{array}[t]{c|cccc}
\comp&1'&a&b&c\\
\hline
1'&1'&a&b&c\\
a&a&1'&c&b\\
b&b&c&1'&a\\
c&c&b&a&1'
\end{array}$};

\tikzset{vertex/.style = {shape=circle,draw,minimum size=1.5em}}
\tikzset{edge/.style = {->}}
\node[vertex](x) at (0,0) {$1$};
\node[vertex](y) at (4,0) {$2$};
\node[vertex](w) at (2,3) {$3$};
\node[vertex](z) at (2,1) {$0$};
\draw[thick,<->] (x) to (y);
\draw[thick,<->] (x) to (z);
\draw[thick,<->] (x) to (w);
\draw[thick,<->] (y) to (z);
\draw[thick,<->] (y) to (w);
\draw[thick,<->] (w) to (z);
\node[below] at (2,0) {$a$};
\node[left] at (1,1.5) {$c$};
\node[above] at (1,.5) {$b$};
\node[left] at (2,2){$a$};
\node[above] at (3,.5){$c$};
\node[right] at (3,1.5) {$b$};

\end{tikzpicture}
\]
This algebra happens to be associative \up(hence a relation algebra\up).  The
only consistent triples of non-identity atoms are the permutations of $(a, b,
c)$.   If we restrict the base to a set of three elements, say $\set{1,2,3}$, we
obtain a different qualitative representation, no longer a strong representation
because although $1'=c;c$ and $(2,2)$ is in the representation of $1'$, there is
no point $v$ in the base $\set{1,2,3}$ such that $(2, v)$ and $(v, 2)$ are in
the representation of $c$. This relation algebra can have no qualitative
representation, nor even a feeble representation, on a base of more than four
points, because it is impossible to colour the edges of $K_5$ using three
colours, $a, b, c$,  while avoiding triangles with two edges of the same colour.  In fact the two qualitative representations just mentioned are the only
qualitative representations of this relation algebra, up to base isomorphism. 

\item\label{ex:ngqr}
Our next example is a non-associative algebra\nb{R1, R3} which does not
have a qualitative representation, though it has a feeble one. Its atoms are
$\set{e,  e', a, \cv{a}}$, \/ $1'=e+e'$ and composition is given by 
\[
\begin{array}{c|cccc}
;&e&e'&a&\cv{a}\\
\hline
e&e&0&a&0\\
e'&0&e'&0&\cv{a}\\
a&0&a&a&a+\cv{a}+e\\
\cv{a}&\cv{a}&0&a+\cv{a}+e'&\cv{a}
\end{array}
\]
By Lemma~\ref{lem:na} this is the atom structure of a non-associative algebra
\up(not associative because $a= a;a=(a;e');a \neq a;(e';a)=a;0=0$\up).  A feeble representation $\theta$ on the  points $0, 1$ is shown below
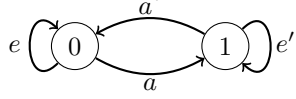
\[
\begin{tikzpicture}
\tikzset{vertex/.style = {shape=circle,draw,minimum size=1.5em}}
\tikzset{edge/.style = {->}}
\node[vertex] (x) at (0,0) {$0$};
\node[vertex] (y) at (2,0) {$1$};
\draw [edge,thick] (y) to[bend right] (x);
\draw [edge,thick] (x) to[bend right] (y);
\node [above] at (1,0.3) {$\cv{a}$};
\node [below] at (1,-0.3) {${a}$};
\node  at (-0.8,0) {$e$};
\node at (2.8,0.06) {$e'$};
\draw [thick,->] (y) to [thick,out=45,in=315,looseness=4] (y);
\draw [thick,->] (x) to [thick,out=225,in=135,looseness=4] (x);
  \end{tikzpicture}
  \]
  If $\phi$ were
a  qualitative representation then since $a;a\cdot a\neq 0$ there
would be $x, y, z$ in the base such that $(x, y), (y, z), (x, z)\in a^\phi$ and
$(y, y)\in (1')^\phi$.  Since $1'=e+e'$ either $(y, y)\in e^\phi$ or $(y, y)\in
(e')^\phi$, in the former case $(x, y)\in a^\phi\cp e^\phi\subseteq
(a;e)^\phi=0^\phi$ and in the latter case $(y, z)\in (e';a)^\phi=0^\phi$, in
each case we get a contradiction.

\item  Examples of relation algebras not even possessing feeble representations are provided in the proof of Theorem~\ref{thm:nfa} below.

\item\label{mk}
This example shows that associativity does not suffice to ensure a qualitatively
representable algebra has a strong representation.  Let $\mathcal{K}$ be
McKenzie's non-representable algebra  
\up(see~\cite{McK70}\up). 
It is defined by the following multiplication table for the atoms
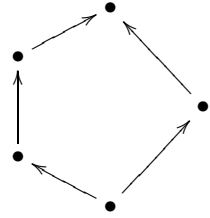
\[
\begin{array}{lr}
\begin{array}[t]{c|cccc}
\comp &1'&a&\cv{a}&b\\
\hline
1'&1'&a&\cv{a}&b\\
 a&a&a&1&a+b\\
 \cv{a}&\cv{a}&1&\cv{a}&\cv{a}+b\\
b&b&a+b&\cv{a}+b&-b
\end{array}
&\hspace{.5in}
\xymatrix@R=.7pc {
&\bullet&\\
\bullet\ar@{->}[ru]&&\\
&&\bullet\ar@{->}[luu]\\
\bullet\ar@{->}[uu]&&\\
&\bullet\ar@{->}[lu]\ar@{->}[ruu]&
}
\end{array}
\]
where $\cv{b} = b$. As McKenzie showed, it is associative but has no strong representation.  Let $\mathbf{N}_5$ be the pentagon
lattice considered as an ordered set, illustrated in the right  above. It is easy to show that boolean 
combinations of the relations $\Id$, $<$, $>$, $\#$ \up(where $\#$ stands for incomparability\up),
form a herd over $\mathbf{N}_5\times\mathbf{N}_5$, and the map
$1'\mapsto\ \Id$, \/$a\mapsto\ <$, \/$\cv{a} \mapsto\ > $, \/$b\mapsto \#$, extends naturally
to a qualitative representation of $\mathcal{K}$. We leave it as an instructive
exercise to prove that no qualitative representation can exist over a set with~4 or
fewer elements. \up(Hint: you have to be able to compose incomparability with itself
and get $<$ and $>$.\up)

\item\label{ex:final} Our final example is an infinite relation algebra, for
  expressing metric constraints on a linearly ordered metric space.  Its elements are finite
  unions of real intervals, e.g. $(2, 5)\cup [6,8]$.  There is one identity atom,
  namely $[0,0]$,  converse is defined by $\cv{(m, n)}= (-n, -m)$, and
  composition is defined by $(m, n);(m', n')=(m+m', n+n')$ for $m>n,\; m'>n'$,
  with similar definitions for closed and semi-open intervals.  A strong
  representation $\theta$ over the real numbers may be obtained by letting $(x,
  y)\in (m, n)^\theta\iff m< y-x< n$, with similar definitions for closed and
  semi-open intervals.  This provides a useful way of expressing metric constraints between points, e.g. the constraint $(x, y)\in ([-3,-2]\cup [2,3])^\theta$ means that the distance between $x$ and $y$ is at least two and not more than three. 
\end{enumerate} 
\end{examples}

\section{Semi-associativity and associativity}

  In order to axiomatise the class of qualitatively representable algebras we might
  start by taking the axioms of non-associative relation algebra and add some
  weakening of the associativity law.  Maddux defines two such weakenings: the
  \emph{semi-associative law} $x\comp (1\comp 1)=(x\comp 1)\comp 1$ and the \emph{weak-asociativity
    law} \cite{Ma82} $ (x\cdot 1')\comp (1\comp 1)=((x\cdot 1')\comp 1)\comp 1$.   In the definition of the qualitative calculus given by \cite{LR04} the identity $1'$ is required to be an atom and for this case the weak associativity law is sure to hold (since $x\cdot 1'$ is either $0$ or $1'$) as shown in  \cite[Section~3]{WestphalHW14}.\nb{R2}   However, when we consider cases where the identity is not an atom, we see that  the algebra of Examples~\ref{examples}.\ref{2-ex} above fails the weak associativity law (and consequently also fails the stronger semi-associativity law) because $(e\comp 1)\comp 1=(e+d)\comp 1=1$ but
  $e\comp (1\comp 1)=e\comp 1=e+d$, so  weak associativity is not valid over  qualitative representations.
  On the other hand, since $(a;b)^\theta\supseteq a^\theta\cp b^\theta$, for any qualitative (or feeble) representation $\theta$ over domain $D$, it follows that
the equation
\[(1\comp x)\comp 1=1\comp (x\comp 1)\] 
is an example of a validity over qualitative (respectively, feeble) representations (in any herd on base $D$, both sides equal $1=D\times D$ if $x\neq 0$ and both sides equal $0$ if $x=0$).

An algebra is \emph{integral} if $x\comp y=0\rightarrow x=0 \text{ or } y=0$.  It is
known for semi-associative algebras that an algebra is integral iff the
identity is an atom \cite[Theorem~4]{Ma90:nec}.  However, in the algebra of  Example~\ref{examples}.\ref{3-ex}
above, the identity is an atom but the algebra is not integral.  It follows that
this algebra is not semi-associative, indeed $d\comp 1=1'+d+f$ but
$(d\comp 1)\comp 1=1$.

Semi-associativity for weak composition was touched upon in~\cite{LR04} where
it was shown that if relations are \emph{serial} (have total domains) then the
weak composition is semi-associative. In fact, in our more general setting where the identity need not be an atom, it suffices to assume that
nonempty relations have pairwise overlapping domains.
\begin{lemma}
Let $\phi$ be a qualitative representation of a non-associative algebra
 $\mathcal{A}$ to a herd $\c S$. Then, the
following are equivalent\up:
\begin{enumerate}
\item If $a, b\in \c A\setminus\set0$ then $\cv{(a^\phi)}\cp b^\phi\neq \varnothing$ \up(nonempty relations have overlapping domains\up),\label{en:over}
\item $\mathcal{A}$ is integral,\label{en:integral}
\item $\mathcal{A}$ is semi-associative and the identity is an atom. \label{en:atom}
\end{enumerate}
\end{lemma}

\begin{proof}
As we noted, in semi-associative
algebras integrality is equivalent to the identity being an atom  so we get \eqref{en:atom} $\Rightarrow$ \eqref{en:integral}.  To prove \eqref{en:integral} $\Rightarrow$ \eqref{en:atom} we will
show that integrality implies semi-associativity in non-associative algebras.
Working backwards, suppose semi-associativity fails in $\mathcal{A}$. By monotonicity of composition in non-associative algebras, since $1'\leq 1$,  we have $x\comp(1\comp 1) = x \comp 1\leq (x\comp 1)\comp 1$, for any $x\in \mathcal{A}$.  Thus, the
failure of semi-associativity is witnessed by some  $0\neq a\in \mathcal{A}$ such that 
$a \comp 1 < (a;1);1$. So, there is  $(a;1);1\geq b\neq 0$ such that $a\comp 1 \cdot b = 0$. 
But then, 
$\cv{a}\comp b \cdot 1 = 0$ by the Peircean law, and so  
$\cv{a}\comp b = 0$ where $\cv{a}, b\neq 0$ and the algebra is not integral.  Thus the equivalence of \eqref{en:integral} and \eqref{en:atom} is true in any non-associative algebra.

To prove the equivalence of
(\ref{en:integral}) and (\ref{en:over}), observe that $a=0\iff\cv a=0$ and  $(\cv a)^\phi\cp b^\phi=\varnothing \iff (\cv a;b)^\phi=\varnothing\iff \cv a;b=0$, using the fact that $c= \cv a;b$ is the minimal solution of $c^\phi\supseteq (\cv a)^\phi\cp b^\phi$.
\end{proof}

Let $\mathcal{A}$ be a non-associative algebra, and $\phi$ be a qualitative representation, i.e. a map
from $A$ to $\pw(D\times D)$ for some set $D$ satisfying the conditions of Definition~\ref{qualrep}.
Consider the following condition on $\phi$: 
\begin{equation}\tag{*}\label{ass}
(a^\phi\cp b^\phi)\cap(c^\phi\cp d^\phi) =\varnothing   \iff 
(a\comp b)\cdot(c\comp d)= 0.
\end{equation}
Intuitively, (\ref{ass}) says that two consistent triangles share a label
(the right-hand side) if and only if
a quadrangle witnessing this fact can be found in the representation (the
left-hand side). Observe that the  right-to-left implication holds for any  qualitative representation $\phi$, since $(x;y)^\phi\supseteq x^\phi\cp y^\phi$.

\begin{theorem}
If there exists a qualitative representation of\/ $\mathcal{A}$ satisfying 
\emph{(\ref{ass})}, then $\mathcal{A}$ is associative. 
\end{theorem}

\begin{proof}  To prove associativity, let $a, b, c, d\in \mathcal{A}$ be arbitrary.  Then
\begin{align*}
((a\comp b)\comp c )\cdot d = 0 &\iff (a\comp b)\cdot(d\comp\cv{c})= 0&&\mbox{by Peircean law for }  \comp\\
 &\iff (a^\phi\cp b^\phi)\cap(d^\phi\cp \cv{c}^\phi)=\varnothing&&\eqref{ass}\\
 &\iff ((a^\phi\cp b^\phi)\cp c^\phi)\cap d^\phi=\varnothing&&\mbox{Peircean law for }\cp \\
 &\iff (a^\phi\cp (b^\phi\cp c^\phi))\cap d^\phi=\varnothing&&\mbox{associativity of }\cp \\
 &\iff(b^\phi\cp c^\phi)\cap(\cv{a}^\phi\cp d^\phi)=\varnothing&&\mbox{Peircean law for }\cp \\
 &\iff (b;c)\cdot(\cv{a};d)=0&&\eqref{ass}\\
 &\iff (a\comp (b\comp c)) \cdot d= 0&&\mbox{Peircean law for } \comp
\end{align*}
Since $d$ is arbitrary it follows that $(a;b);c=a;(b;c)$.
\end{proof}
Observe that the\nb{R1, R3}  McKenzie algebra $\mathbf{K}$  is associative, but its qualitative representation of 
 over the  base $\mathbf{N}_5$ does not satisfy
(\ref{ass}), because 
we have $(\#\cp\mathord{<})\cap(\#\cp\mathord{>}) = \varnothing$ whereas
$(b\comp a)\cdot (b\comp\cv{a}) = b$. However, there is a qualitative
representation of   $\mathbf{K}$  satisfying (\ref{ass}), for example over the base
$D = \{\bot,a_1,b_1,c_1,a_2,b_2,c_2,\top\}$ with $<$ defined
as the transitive closure of $\bot<a_1<b_1<c_1<\top$, 
$\bot<a_2<b_2<c_2<\top$.   The following conjecture remains open: if $\mathcal{A}$ is associative and has a qualitative representation then it has a qualitative representation satisfying \eqref{ass}.

\section{Network Satisfaction Problem}\label{sec:NSP}

\begin{definition}\label{def:network} Let $\c A$ be a non-associative algebra.
A \emph{network} $(N, \lambda)$ over $\c A$ consists of a finite\nb{R3} set $N$ of nodes and a
function $\lambda:(N\times N)\rightarrow \c A$.  A network $(N, \lambda)$
is \emph{consistent} if 
\begin{enumerate}
\item[($a$)] $\lambda(x,x)\leq 1'$, 
\item[($b$)] $\lambda(x, y)\comp \lambda(y, z)\;\cdot\; \lambda(x, z)\neq 0$,
 for all nodes $x, y, z\in N$, \label{en:b}
\item[($c$)] $\lambda(x, y)\cdot\lambda\cv{(y, x)}\neq 0$,
\item[($d$)] $\lambda(x, y)\neq 0$, for all nodes $x, y\in N$.  
\end{enumerate}
A network $(N, \lambda)$  is \emph{path-consistent} (or algebraically closed) if it is consistent and additionally $\lambda(x, y)\comp\lambda(y, z)\geq\lambda(x, z)$, for all $x, y, z\in N$.
An \emph{atomic network}  $(N, \lambda)$ is a network where $\lambda(x, y)$ is
always an atom of $\c A$.   Every consistent atomic network is path-consistent.

A network $(N, \lambda)$
\emph{embeds} into a strong representation $\phi$ if there is a map $\prime$
from $N$ to the base of $\phi$ such that for all $x, y\in N$ we have $(x',
y')\in \lambda(x, y)^\phi$, similarly $(N, \lambda)$ embeds into a qualitative
representation $\theta$ if there is a map $\prime$ from $N$ to the base of a
qualitative representation $\theta$ such that for all $x, y\in N$ we have $(x', y')\in
\lambda(x, y)^\theta$.   A network over $\c A$  is \emph{strongly satisfiable} if it
embeds into some strong representation of $\c A$ and it is \emph{qualitatively satisfiable} if
it embeds into some qualitative  representation of $\c A$.  Clearly, if $(N, \lambda)$
is strongly satisfiable then it is qualitatively satisfiable. 

A strong representation \nb{R1} $\phi$ of the finite relation algebra $\c A$ is
\emph{universal} \nb{R1} if every consistent atomic network embeds into $\phi$.  
\end{definition}
Note that the conditions (c) and (d) for a
consistent network follow from (a) and (b) (see~\cite[Lemma~7.2]{HH:book}), so 
we could have left them out of the definition, but since they are 
naturally expected nonetheless, we decided to keep them in.
Our definition of a \emph{consistent network} coincides with what 
Hirsch and Hodkinson~\cite{HH:book} call a \emph{network}. We believe this 
less restrictive definition of a network (essentially, just as a 
labelled complete directed graph) to be more convenient in the present context
as it is closer to 
the standard terminology used in the area of qualitative calculi as well as in 
constraint satisfaction.

\begin{remark}\label{remark}
Conisider\nb{R2} the non-associative algebra of Example~\ref{examples}.\ref{2-ex}.  It has a qualitative representation, but since it is not associative it can have no strong representation.  The network shown in Example~\ref{examples}.\ref{2-ex} is qualitatively satisfiable but not strongly satisfiable.

Let $\c A$ be an  atomic relation algebra. If $\c A$ has a universal representation then any network is qualitatively satisfiable iff it is strongly satisfiable.
It is known that the point and interval algebras have universal representations \cite{LaMa94} and it follows from results in \cite{BW11} that RCC8 has a universal representation, hence  for these three relation algebras a network is strongly satisfiable iff it is qualitatively
satisfiable\nb{R1}. 
\end{remark}

\begin{lemma}\label{lem:net}
Let $\c A$ be a finite non-associative algebra.  $\c A$ has a qualitative
representation if and only if there is a consistent atomic network $(N, \lambda)$
over $\c A$ such that for each consistent triple of atoms $(a, b, c)$ of $\c A$
there are nodes $x, y, z\in N$ such that $\lambda(x, y)=a,\; \lambda(y, z)=b$
and $\lambda(x, z)=c$. 
\end{lemma}
\begin{proof}
Let $\phi$ be a qualitative representation over base $D$.  For each $x, y\in D$ let
$\lambda(x, y)$ be the (unique) atom $a$  such that $(x, y)\in a^\phi$, such an
atom must exist since $\c A$ is finite.  Clearly $(D, \lambda)$ is  consistent and atomic, though it might be infinite.\nb{R3}  Furthermore, if $a, b, c$ are atoms such that $a\comp b\geq c$
then there must be $x, y, z$ with $\lambda(x, y)=a,\;\lambda(y, z)=b$ and
$\lambda(x, z)=c$, by Lemma~\ref{lem:xyz}.  Hence there is a finite subset $D_0\subseteq D$ such that the restriction of $\lambda$ to $D_0$ is atomic, consistent and therefore an atomic network,  and still witnesses all triples of atoms, 

Conversely, if $(N, \lambda)$ is a consistent atomic network witnessing all
consistent triples of atoms then the binary relation $\sim$ over the nodes of $N$ defined by $x\sim y\iff N(x, y)\leq 1'$ is easily seen to be an equivalence relation, indeed a congruence.  The equivalence class of a node $x$ is denoted $[x]$.   We may define a qualitative representation $\phi$ whose base is the set of all $\sim$-equivalence classes, by
$$
a^\phi = \{([x], [y])\colon  N(x, y)\leq a\}
$$ 
for $a\in\c A$.  Since $\sim$ is a congruence $\phi$ is well-defined, since edges are
labelled by atoms $\phi$ respects the boolean operators, since the network is
consistent it is clear that $\phi$ respects  the converse
operator, and since $N(x, y)\leq 1'\iff [x]=[y]$ the identity is correctly represented.  By Lemma~\ref{lem:xyz}, $\phi$ is a  qualitative representation.
\end{proof}

\begin{lemma}\label{lem:cubic}
If $\c A$ is a finite, qualitatively representable, atomic non-associative algebra
then $\c A$ has a qualitative representation with at most $3|At(\c A)|^3$ points in its
base.
\end{lemma}
\begin{proof}
By the previous lemma, if $\c A$ is qualitatively representable then there is
a consistent atomic network witnessing all consistent triples of atoms.  There
are at most $|At(\c A)|^3$ such triples, so $3|At(\c A)|^3$ points suffice to
witness them all.   The atomic network defined by restricting to this set of up to $3|At(\c A)|^3$ points is consistent and still witnesses all consistent triples of atoms, hence it defines a qualitative representation of the required size, by the proof of the right to left implication of  Lemma~\ref{lem:net}.
\end{proof}

The upper bound of Lemma~\ref{lem:cubic} seems to overestimate the necessary
size of a qualitative representation rather largely. Although we will not try to provide a
sharper bound here, we will present an illustrative example. Consider 
RCC5: a version of RCC8 with no distinction between ``tangential'' and ``non-tangential''
connectedness. Its composition table is
$$
\begin{tabular}{c|ccccc}
$\comp$     & $\id$     & $\ep$ & $\cv{\ep}$ & $\pi$ & $\de$ \\ 
\hline
$\id$     & $\id$     & $\ep$ & $\cv{\ep}$ & $\pi$ & $\de$ \\ 
$\ep$      & $\ep$      & $\ep$ & $1$ & $\ep\join\pi\join\de $ & $\de$ \\ 
$\cv{\ep}$ & $\cv{\ep}$ & $\id\join\ep\join\cv{\ep}\join\pi$ & $\cv{\ep}$ 
& $\cv{\ep}\join\pi$ & $\cv{\ep}\join\pi\join\de$ \\ 
$\pi$      & $\pi$      & $\ep\join\pi$ & $\cv{\ep}\join\pi\join\de$ & $1$ & $\cv{\ep}\join\pi\join\de$ \\ 
$\de$      & $\de$      & $\ep\join\pi\join\de$ & $\de$ & $\ep\join\pi\join\de$ & $1$ \\  
\end{tabular}
$$
with $\ep$, $\pi$ and $\de$ interpreted intuitively as proper part, proper overlap, and
disjointness relations, respectively. This algebra has a qualitative
representation over a base consisting of just eleven ``regions'', namely, the
following subsets of $\{1,\dots,7\}$:
\begin{align*}
A &= \{1,6\}&  B &= \{1,2,3,5,6\}& C &= \{1,2,6\}\\
D &= \{1,2,3,5,6,7\}& E &= \{1\}& F &= \{1,2,3,5\}\\
G &= \{2,3\}& H &= \{1,4,6\}& I &= \{1,3\}\\
J &= \{1,2,3,4,5\}& K &= \{4,5,6\}
\end{align*}
with the relations $\ep$, $\pi$ and $\de$ mapped, respectively, to proper subset relation,
nonempty symmetric difference relation, and empty intersection relation. 

%
%
%
%
%
%
%
%

We conjecture that our representation on the 11 regions $A, \ldots, K$ is the smallest number of regions possible for a qualitative representation of RCC5.

\begin{remark}
The study of \emph{syllogistics} \cite{prioranalytics}\nb{R3} can be described using RCC5.   Namely, setting:
\begin{align*}
\text{Every } S \text{ is } P &\text{\quad if{}f\quad } (S,P)\in \id+\ep\\
\text{Some } S \text{ is } P &\text{\quad if{}f\quad } (S,P)\in \id+\ep+\cv{\ep}+\pi\\
\text{No } S \text{ is } P &\text{\quad if{}f\quad } (S,P)\in \de\\
\text{Some } S \text{ is not } P &\text{\quad if{}f\quad } (S,P)\in\cv{\ep}+\pi+\de
\end{align*}
we obtain a faithful interpretation of traditional logic of 
\emph{categorical propositions}. 
\end{remark}

We turn to questions of computational complexity.  We say that a finite atom structure has a qualitative representation if its complex algebra has one.
We begin with a result that contrasts with the corresponding question for
strong representations which is known to be undecidable \cite{HH7}. 

\begin{theorem}\label{thm:NP}
The problem of determining whether a finite atom structure has a qualitative
representation is {\bf NP-complete}. 
\end{theorem}

\begin{proof}  
If  a finite atom structure with $n$ atoms has a qualitative representation then,
by Lemma~\ref{lem:cubic}, it has a qualitative representation of size at most $3n^3$.  Hence a non-deterministic algorithm may simply construct a base with up to $3n^3$ points and guess an atom between each pair of points, then check to see if the resulting network is consistent and that every triple of atoms is witnessed (see Lemma~\ref{lem:net}).  Since the run-time of this non-deterministic algorithm is bounded \nb{R1} by a polynomial function, we conclude that the qualitative representation problem is in {\bf NP}.

\medskip

For {\bf NP}-hardness, we reduce the \emph{3-colour graph vertex problem} for finite graphs.  Let $G=(V, \mathscr{E})$ be a finite graph with vertices $V$ and directed edges $\mathscr{E}$, where this edge set is symmetric and irreflexive.  Cases where $\E=\varnothing$ are trivially $1$-colourable, so we assume $\E\neq\varnothing$.   If we extend $G$ by adding some isolated vertices to $V$ it will not affect the $3$-colourability of the graph, so we assume that $G$ has an independent set of size $5$ (formally, there are $v_0, v_1, \ldots, v_4\in V$ such that $(v_i, v_j)\not\in E$, for $i, j<5$)\nb{R1} and a triangle containing one edge and two non-edges (there are $u, v, w\in V$ such that $(u, v)\in E$ but $(u, w), (v, w)\not\in E$). Next, we extend $G$ to $G^\infty=(V^\infty, \E^\infty)$, by adding a single node connected to all the nodes of $V$, i.e. $V^\infty=V\cup\set\infty,\; \E^\infty=\E\cup\set{(\infty, u), (u, \infty):u\in V}$, where $\infty\not\in V$.      It is clear that $G^\infty$ is 4-colourable if and only if $G$ is $3$-colourable.   Because $G$ has an independent set of size $5$, any $4$-colouring of $G^\infty$ must include non-adjacent nodes of the same colour, and because of the triangle with one edge and two non-edges any colouring must include non-adjacent nodes of different colours.

We define a non-associative
atom structure $(S, E, \cv{ }, C)$ as follows.    The set of atoms is
\[ S=\set{1'}\cup S_{GG}\cup S_{CC}\cup S_{GC}\cup S_{CG}\]
where $S_{GG}=\set{s_{uv}: (u, v)\in\E^\infty}\cup\set{\g}$ (graph atoms, here $\g$ is a symbol not appearing in $V^\infty$, used for non-edges), \/ $S_{CC}=\set{a, b, c}$ (colouring atoms), \/ $S_{GC}=\set{\y,\n}$ (used to map graph nodes to colours) and $S_{CG}=\set{\cv\y, \cv\n}$ (the converses of $S_{GC}$).  The identity is an atom $E=\set{1'}$, and all atoms are self-converse except 
$\cv{s_{uv}}=s_{vu}$ and the converses of $\y, \n$ are $\cv\y, \cv\n$ respectively.  
[The intention here is that atoms  $s_{uv}$ will be used to encode the edges $(u,v)$ of the graph $G^\infty$, \/  $\g$ (``gap'') corresponds to non-edges and the atoms $a, b, c\in S_{CC}$ will
encode the undirected edges of a graph with no more than four nodes (see Examples~\ref{examples}.\ref{ex:4}) and these nodes will represent distinct colours.
$\sfy,\cv\y$ (``yes'') and $\sfn,\cv\n$ (``no'') will be used to connect this
set of up to four nodes to the nodes of $G^\infty$ while encoding a  legitimate
$4$-colouring.  A schematic of the desired qualitative representation is shown
in Figure \ref{fig:qualrepNP}.\nb{R2,R3.}] 
\begin{figure}\begin{tikzpicture}
\begin{scope}[xscale = 1.1,yslant =0.3]
\node (infty) [vertex] at (0,3) [label = $\infty$]{}; 
\newdimen\rrr
\rrr=1.4cm
\foreach \n in {0,1,2,3,4,5} {
\node (\n) [vertex] at (60*\n:\rrr) [label=60*\n:\n] {};
}
\foreach \n in {1,2,3,4,5} {
    \foreach \m in {0,1,2,3,4,5} {
    \draw [lightgray] (\m) -- (\n);
    }
}
\draw [thick, <->] (0) -- (5);
\draw [thick, <->] (5) -- (4);
\node at (1.5,-.9) {$s_{0,5}$};
\node at (.25,-1.6) {$s_{4,5}$};
\foreach \n in {1,2,4,5} {
\draw [<->] (infty) -- (\n);
}
\draw [bend angle = 20,<->] (infty) [bend left] to (0);
\draw [bend angle = 25,<->] (infty) [bend right] to (3);
\end{scope}
\begin{scope}[xshift = 5cm,yshift = -1cm,xscale = .7,yslant =0.3]
\node[vertex](x) at (0,0) {};
\node[vertex](y) at (3,0) {};
\node[vertex](w) at (1.5,2.25) {};
\node[vertex](z) at (1.5,.75) {};
\draw [thick] (x) -- (y) -- (w) -- (z) -- (y);
\draw [thick] (z) -- (x) -- (w);
\draw [-<] (x) -- (-2,1);
\draw [-<] (x) -- (-2,1.4);
\draw [-<] (x) -- (-2,.6);
\draw (x) -- (-.5,.4);
\draw (x) -- (-.5,0.1);
\draw (x) -- (-.5,0.2);
\draw [-<] (y) -- (-.5,1);
\draw [-<] (y) -- (-.4,1.1);
\draw [-<] (y) -- (-.5,0.8);
\draw (y) -- (2,0.1);
\draw (y) -- (2.5,0.3);
\draw (y) -- (2.5,0.4);
\draw [-<] (z) -- (-3,1.8);
\draw [-<] (z) -- (-3,2.1);
\draw [-<] (z) -- (-3,2.4);
\draw (z) -- (1,1.1);
\draw (z) -- (1,.8);
\draw (z) -- (1,.7);
\draw [-<] (w) -- (-1,2);
\draw [-<] (w) -- (-2,2.6);
\draw [-<] (w) -- (-3,4.2);
\draw (w) -- (1,2.1);
\draw (w) -- (1,2.4);
\draw (w) -- (1,2.25);
\node at (-1,2.7) {$\textsf{n}$};
\node at (-1,2.2) {$\textsf{n}$};
\node at (-1.1,3.6) {$\textsf{y}$};
\node at (1.6,-.2) {$a$};
\node at (2.6,1.2) {$b$};
\node at (.7,1.6) {$c$};
\node at (1.8,1.2) {$a$};
\node at (1.03,.22) {$b$};
\node at (2.03,.62) {$c$};
\end{scope}
\end{tikzpicture}
\caption{A schematic diagram of the intended qualitative representation for $(S,
  E, \cv{ }, C)$, when $G$ is $3$-colourable.  The graph $G^\infty$ is on the
  left, with vertex set $\{0,1,2,3,4,5,\infty\}$ and edges shown in black.  An
  independent set can be found on the five vertices $\{0,1,2,3,4\}$ and a triple
  of vertices with one edge can be found at $\{0,1,5\}$.  The vertex $\infty$ is
  adjacent to all vertices and nonedges are shown in grey (and would be labeled
  by $\mathsf{g}$).  On the right, there is a second tetrahedral graph on $4$
  vertices (with names $\gamma_1,\gamma_2,\gamma_3,\gamma_4$ omitted) and with
  undirected edges $a,b,c$ as in Example~\ref{examples}.\ref{ex:4}.  Edges
  connecting the graph $G^\infty$ to these vertices are labelled $\textsf{y}$ or $\textsf{n}$, according to the colour of the vertex of $G^\infty$.}\label{fig:qualrepNP}
\end{figure}
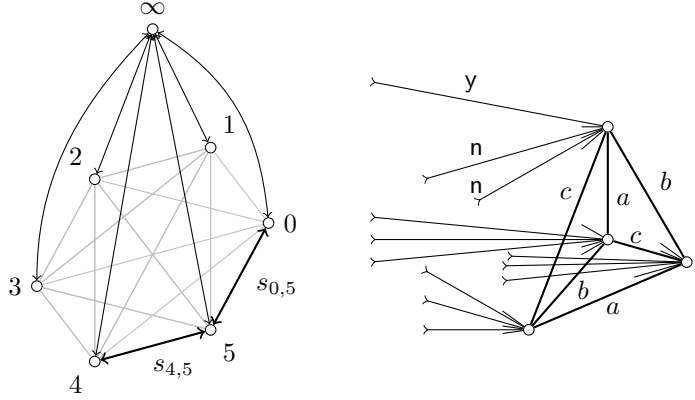

All Peircean transforms of the following triples of atoms are forbidden.
\begin{enumerate}
\renewcommand{\theenumi}{\Roman{enumi}}
\item $(1', x, y)$ where $x\neq y$, (identity law)\label{forb:1'}
\item\label{forb:GC} $(\alpha, \beta, \gamma)$, where $\alpha\in S_{IJ},\;\beta\in S_{J'K'}$ and $\gamma\in S_{I^*K^*}$, unless $I=I^*, \; J=J',\; K'=K^* \;\;(\in\set{G, C})$ \/ (types must match),
\item \label{forb:CC}$(\alpha, \alpha, \beta)$ where $\alpha, \beta\in S_{CC}$, (see Examples~\ref{examples}.\ref{ex:4}, only permutations of the triple $(a, b, c)$ are consistent for $S_{CC}$ atoms),

\item\label{forb:GG} $(s_{uv}, s_{v'w'},\alpha )$ where $(u, v), (v', w')\in\E^\infty$, unless $v=v'$ (node indices must match),
  
\item\label{forb:GG2} $(s_{uv}, s_{vw}, \g)$ where $(u, v), (v, w) \in\E^\infty$ and either $(u, w)\in\E^\infty$ or $u=w$, ($\g$ only allowed on non-edges),
\item\label{forb:GG3} $(s_{\infty u}, \alpha, \g)$ where $u\in V,\;\alpha\in S_{GG}$ ($\infty$ adjacent to all other nodes),
\item\label{forb:1 col}  $(\y, \alpha, \y)$ where $\alpha\in S_{CC}$ (only one colour per node)
\item\label{forb:last} $(s_{uv}, \y, \y)$ where $(u, v)\in \E^\infty$ (adjacent nodes have different colours)
\end{enumerate}
The set $C$ of consistent triples of atoms consists of all triples not forbidden by \eqref{forb:1'}--\eqref{forb:last}, above.

The reduction maps $(V, \mathscr{E})$ to the  atom
structure $(S, E, \cv{ }, C)$ just defined. 
We check that this reduction is correct.  Suppose $\rho$ is a 4-colouring of the
vertices of $G^\infty$ such that adjacent nodes have different colours,  there are non-adjacent nodes of the same colour and non-adjacent nodes of different colours.  Let the
colours be $\gamma_1, \gamma_2, \gamma_3, \gamma_4$.  We define an atomic network $(N, \lambda)$ where
$N=V^\infty\cup\{\gamma_1, \gamma_2, \gamma_3, \gamma_4\}$.  For the labelling $\lambda$ we  
 let $\lambda(x, x)=1'$ (all $x\in N$), \/ $\lambda(u,v)=s_{uv}$ for $(u,v)\in \mathscr{E}^\infty$, \/ $\lambda(u,v)=\g$ if $u\neq v\in V^\infty,\; (u,v)\notin \mathscr{E}$.  
 As in the strong representation for Examples~\ref{examples}.\ref{ex:4}, let $\lambda(\gamma_1, \gamma_2)=\lambda(\gamma_3, \gamma_4)=a,\;\lambda(\gamma_1, \gamma_3)=\lambda(\gamma_2, \gamma_4)=b,\;\lambda(\gamma_1, \gamma_4)=\lambda(\gamma_2, \gamma_3)=c$ and $\lambda(\gamma_j, \gamma_i)=\lambda(\gamma_i, \gamma_j)$ for $1\leq i<j\leq 4$. Finally, for each $u\in V^\infty$ and for $1\leq i\leq 4$, let $\lambda(v,\gamma_i)=\y$ if $\rho(v)=\gamma_i$ else (if $\rho(v)\neq\gamma_i$)\/ $\lambda(v,\gamma_i)=\sfn$.  The atoms $\cv\sfy$ and $\cv\sfn$ are used to label the converse edges for edges labelled $\y, \n$ respectively.  It is a routine check that this defines a consistent atomic network.  Since there are non-adjacent nodes of the same colour the triple $(\g, \y, \y)$ is witnessed and since there are non-adjacent nodes of different colours $(\g, \y, \n)$ and $(\g, \n, \y)$ are also witnessed.  It is easily checked that all other consistent triples of atoms are also witnessed hence, by Lemma~\ref{lem:net}, the complex algebra over $(S, E, \cv{ }, C)$ is qualitatively representable.

For the converse, let $(N, \lambda)$ be any complete, consistent atomic network witnessing  each consistent triple and containing no forbidden triples for the atom structure $(S, E, \cv{ }, C)$: in other words, $(N, \lambda)$ provides a qualitative representation of $(S, E, \cv{ }, C)$.  We show that $G^\infty$ is $4$-colourable, hence $G$ is $3$-colourable.  By forbidden triple \ref{forb:GC}, a triple $(\alpha, \beta, \gamma)$ is forbidden if exactly one or all three of $\alpha, \beta, \gamma$ belongs to $S_{GC}\cup S_{CG}$.  Hence the edges labelled by atoms in $S_{GC}\cup S_{CG}=\set{\y, \cv\y, \n, \cv\n}$ form a complete bipartite graph on $N$, say $N$ is the disjoint union of $N_1$ and $N_2$.  By forbidden triple \ref{forb:GC} again, every edge with source in $N_1$ (say) and target in $N_2$ has a label in $S_{GC}$, every edge with source in $N_2$ and target in $N_1$ has a label in $S_{CG}$,  edges with source and target within $N_1$ have label in $\set{1'}\cup S_{GG}$ and edges with source and target in $N_2$ have label in $\set{1'}\cup S_{CC}$.

Now we show that $|N_1|=|V^\infty|$ and that the labels $s_{uv}$ define a graph on $N_1$ isomorphic to $G^\infty$.  Let $u\in V$ and let $\infty', u'\in N_1$ denote the source  and target of some edge labelled by $s_{\infty u}$.   For each $x\in N_1\setminus\set{\infty'}$ the edge $(\infty', x)$ must be labelled by $s_{\infty v}$ for some $v\in V$ by forbidden triples \ref{forb:GC}, \ref{forb:GG} and \ref{forb:GG3}, and  if $(\infty', y)$ is also labelled by $s_{\infty' v}$ then $x=y$ by forbidden triple \ref{forb:GG2}.  Hence there is a bijection $':V^\infty\rightarrow N_1$ mapping $\infty$ to $\infty'$ and mapping $v\in V$ to the unique $v'\in N_1$ such that $(\infty', v')$ is labelled $s_{\infty v}$.
 So we may assume that $N_1$ is identical to the set $V^\infty$, and that the label of each edge $(\infty,w)$ (for $w\in V$) is $s_{\infty w}$. By forbidden triples \ref{forb:GG} and \ref{forb:GG2} we see that the label of the edge  $(u,v)$ is $s_{uv}$ if $(u,v)\in \mathscr{E}$, \/ $1'$ if $u=v$ and $\g$ otherwise.
By forbidden triple \ref{forb:CC},  $N_2$ cannot have more than four points,  (we saw in Examples~\ref{examples}.\ref{ex:4} that $N_2$ has either three or four points).

Let $u\in N_1$.
As $(\y, \y, \alpha)$ is always forbidden by \ref{forb:1 col}, there is at most one edge labelled $\sfy$ leaving  $u$ and because $(s_{vu},\sfy,\sfn)$ is consistent  and the edge labelled $s_{vu}$ is unique (where $(u, v)\in \E$), there is exactly one edge labelled $\sfy$ leaving $u$.  Thus, we may define a map $\rho:V^\infty\rightarrow N_2$ by letting $\rho(u)$ be the unique element of $N_2$ such $\lambda(u, \rho(u))=\sfy$, for each $u\in V^\infty=N_1$.  Since $(s_{uv}, \sfy, \sfy)$ is forbidden by \ref{forb:last} whenever $(u, v)\in  \mathscr{E}^\infty, \; \rho$ is a valid 4-colouring of $G^\infty$.
\end{proof}

\begin{theorem}
Let $\c A$ be a finite non-associative algebra.  The network qualitative satisfaction problem over $\c A$  is in {\bf NP}.
\end{theorem}
\begin{proof}
For each consistent triple of atoms $t$, let $(T_t, \lambda_t)$ be a partially labelled atomic network with three nodes,  witnessing the triple of atoms.  Assume for distinct consistent triples $t, s$ that $T_t\cap T_s=\varnothing$.
Given a network $(N, \lambda)$, take the disjoint union of $(N, \lambda)$ and the disjoint partial triangles $(T_t, \lambda_t)$ as $t$ ranges over consistent triples of atoms.  Then, non-deterministically guess all unlabelled edges and for each edge $(x, y)$ of $N$, guess an atom below $\lambda(x, y)$.  Finally check that the resulting atomic network is consistent. 
By  Lemma~\ref{lem:net}, this correctly tests qualitative satisfiability and runs non-deterministically in polynomial time.
\end{proof}
A traditional approach to solving the network satisfaction problem (for strong representations) is to refine a given network  to a \emph{path consistent} network  on the same nodes.  This path-consistent refinement may be computed deterministically in cubic time.  If the refined network is inconsistent (has a zero label) then the original network is unsatisfiable.  There are some algebras where the converse also holds --- if the refined network is consistent then the original network must be satisfiable --- for example, path-consistency entails satisfiability for networks over the \emph{Point Algebra} \cite[Theorem 5]{VK89}.  However, for many relation algebras the converse fails, for example, there are known path-consistent but unsatisfiable networks over the Allen Interval Algebra \cite[Figure~5]{All83}.   The Allen Interval algebra possesses a \emph{universal representation} (see Definition~\ref{def:network}), so the satisfiability of a network may be tested along the lines of the proof of the preceding theorem, by non-deterministically picking an atom below the label of each edge of the network and checking the consistency of the resulting atomic refinement.   Such satisfiability checkers may run faster if a network is first refined to a path-consistent network, before the non-deterministic choice of atoms is made.
 For algebras which do not have universal representations however, it is not sufficient to find a consistent atomic refinement of a network, for example, there are consistent atomic networks over the interval-with-duration calculus INDU, not satisfiable in any strong representation (see our remarks in Section~\ref{sec:1.1}).  Indeed there are finite  relation algebras for which the network satisfaction problem is undecidable \cite{Hir:undec}.
  
Nevertheless, for qualitative representations the non-deterministic algorithm given in the preceding theorem is necessary and sufficient for qualitative representability\nb{R3}.  Hence in general the network qualitative satisfaction problem can have much lower complexity than the network satisfaction problem.

\begin{theorem}
The problem of determining whether an equation is valid over qualitative representations is {\bf co-NP-complete}.
\end{theorem}
\begin{proof}
Let $t(\bar x)=s(\bar x)$ be an equation that fails in some herd $\c S$, where $\bar x=(x_0, \ldots, x_{k-1})$ is a finite tuple of variables. Then $t(\bar a)^{\c S}\neq s(\bar a)^{\c S}$ holds for some tuple $\bar a=(a_0, \ldots, a_{k-1})$ of relations in the herd $\c S$ and there is a pair of points $x, y$ from the base of $\c S$ such that $(x, y)$ belongs to one but not the other of  $t(\bar a)^{\c S}$ and $s(\bar a)^{\c S}$.  Let $X$ be a finite subset of the base of $\c S$ including $x$ and $y$ and also including, for each subterm $p;q$ occuring in either $s$ or $t$, three points $u, v, w$ such that $(u, v)\in p$ and $(v, w)\in q$, provided such points exist in the base of $\c S$.  The size of $X$ is at most three times the length of the equation $t(\bar x)=s(\bar x)$.  Let $\c S\restr{X}$ be the herd on the base $X$ consisting of the relations $\set{r\cap(X\times X):r\in\c S}$.  Lemma~\ref{lem:xyz} and a simple induction shows that $(u, v)\in p(a_0, \ldots, a_{k-1})^{\c S}\iff (u, v)\in p(a_0\cap(X\times X),\ldots, a_{k-1}\cap(X\times X))^{\c S\restr{X}}$, for any $u, v\in X$ and any subterm $p(\bar x)$ of either $t(\bar x)$ or $s(\bar x)$.  Hence $(x, y)$ belongs to one but not the other of $t(a_0\cap(X\times X), \ldots, a_{k-1}\cap(X\times X))^{\c S\restr X},  \; s(a_0\cap(X\times X), \ldots, a_{k-1}\cap(X\times X))^{\c S\restr X}$  and the equations $t(\bar x)=s(\bar x)$ fails in a herd on a base of size at most three times the length of the equation.  Thus, the failure of the equation may be tested non-deterministically by choosing a base of size at most three times the length of the equation, guessing which pairs of points belong to each of the relations $x_0, \ldots, x_{k-1}$ and verifying that the equation fails.  This proves that the validity of equations problem is {\bf co-NP}.

Validity is {\bf co-NP-hard} because the validity problem for propositional formulas reduces to it.
\end{proof}
The reader has probably guessed, or knows  already, that the corresponding problem for strongly representable relation algebras is much harder: the equational theory of strongly representable relation algebras is undecidable \cite{Tar41}.

\section{Feeble representations}\label{sec:feeble}
We mentioned earlier that \cite[Definition 3]{LR04} only requires $(a;b)^\phi\supseteq a^\phi\cp b^\phi$ in their definition of weak composition, and do not insist on minimality subject to that (although elsewhere in their paper they imply a definition which accords with our definition of qualitative representation, see Definition~\ref{qualrep}.\ref{en:abc}).  In this section we investigate this weaker type of representability.  Although\nb{R1} qualitative representations are of greater relevance to much of the literature in qualitative calculi, feeble representations are more directly related to research on the binary constraint satisfaction problem, where for any constraints $a, b$ and a solution for variables  $x, y, z$,  we may infer from $a(x, y), \; b(y, z)$ that $(a\circ b)(x, z)$ holds, but there is no requirement, given $(a\circ b)(x, z)$, that there should be three variables $x', y', z'$ such that $a(x', y')$ and $b(y', z')$.

Let $\c A$ be a non-associative algebra and let $\phi:\c A\rightarrow\wp(D\times D)$ (some base $D$) respect the boolean operators, represent $1'$ as the identity and respect the converse operator.  If $(a;b)^\phi\supseteq a^\phi\cp b^\phi$ (for all $a, b\in\c A$) then $\phi$ is called a feeble qualitative representation, or simply a feeble representation, of $\c A$.
The proof of the following lemma is very similar to the proof of Lemma~\ref{lem:net}.
\begin{lemma}\label{lem:feeble}
Let $\c A$ be a finite non-associative algebra. $\c A$ has a feeble representation if and only if there is a consistent atomic network $(N, \lambda)$ over $\c A$ such that for every atom $a$ of $\c A$ there are $x, y\in N$ with $\lambda(x, y)=a$.
\end{lemma}
In a feeble representation of a finite algebra, instead of 
requiring a witness for each consistent triple of atoms, we only require a 
witness for each each single atom, while avoiding any forbidden triple of atoms.

One obvious shortcoming of this notion is that a feeble representation 
of an algebra $\c A$ does not determine $\c A$, because some consistent 
triples may be absent in the representation.   Indeed, we saw that the herd on a base of two points shown in the second part of Example~\ref{examples}.\ref{2-ex} provides a feeble representation of two non-isomorphic non-associative algebras: the algebra of Example~\ref{examples}.\ref{2-ex} and  the algebra of Example~\ref{examples}.\ref{ex:ngqr}.    Unfortunately, the complexity of the feeble representation problem remains the same as that of the qualitative representation problem.

\begin{theorem}
The problem of determining whether a finite atom structure has a feeble representation is {\bf NP-complete}.
\end{theorem}
\begin{proof}
If a finite atom structure with $n$ atoms has a feeble representation then we may restrict the base to a set of at most $2n$ points so that all atoms are still witnessed in the restriction, hence the atom structure has a feeble representation of size at most $2n$.  Thus, a non-deterministic algorithm may guess an atomic labelling over a set of at most $2n$ points and check if the labelling defines a feeble representation.  So the feeble representation problem is in ${\bf NP}$.

For {\bf NP}-hardness, we reduce the \emph{Monochromatic Triangle} problem for
finite graphs.  Let $G=(V, \mathscr{E})$ be a finite graph with vertices $V$ and
directed, irreflexive, symmetric edges $\mathscr{E}$ (i.e. $(u,v)\in\E\iff (v,u)\in\E$).  The \emph{Monochromatic Triangle} problem asks if
there is a $2$-colouring of the edges, that is, a symmetric function $\rho:\E\rightarrow\set{\r, \sfb}$ such that $\rho$ is not constant over the three edges of any triangle of the graph.  This problem is
known to be {\bf NP-complete}, see \cite[p.~191]{Garey:1979:CIG:578533}.   Any complete graph with six or more nodes is a no instance, hence we will only consider instances of this problem containing at least one non-edge.

  Given a symmetric, irreflexive graph $G=(V,\E)$, we define a non-associative
atom structure $(S, E, \cv{ }, C)$ as follows.   
\[
 S=\set{1'}\cup\set{\x}\cup S_{GG}\cup S_{\infty G}\cup S_{G\infty}\]
where $S_{GG}=\set{c_{uv}:c\in\set{\r, \sfb},\; (u, v)\in \E}\cup\set{\g},\; S_{\infty G}=\set{\p_u, \q_u:u\in V}$ and $S_{G \infty}=\set{\cv\p_u, \cv\q_u:u\in V}$. The identity is an atom $E=\set{1'}$ and converse is defined by $\cv{c_{uv}}=c_{vu}$ (for any $c\in\set{\r, \sfb}$ and $(u,v)\in\E$) ,  the converses of $\p_u, \q_u$ are $\cv \p_u, \cv\q_u$ and $1', \x,\g$ are  self-converse.  
Let $C$ consist of all triples of atoms except for Peircean transforms of the following forbidden triples of atoms.  
\begin{enumerate}
\renewcommand{\theenumi}{\Roman{enumi}}
\item $(1', x, y)$, where $x\neq y$,
\item\label{forb:y} $(\x, \x, \x)$ and $(\x, \alpha, \beta)$, where $\x\not\in\set{\alpha, \beta}$,
\item \label{forb:S} $(\alpha, \beta, \gamma)$ where $\alpha\in S_{IJ},\;\beta\in S_{J'K}$ and any atom $\gamma$, unless $J=J'$,
\item \label{forb:indices2}$(\p_u, c_{u_1v}, \p_{v_1}),\; (\q_u, c_{u_1v}, \q_v)$, any $c\in\set{\r, \sfb}$ any $u, u_1, v, v_1\in V$, unless $u=u_1$ and $v=v_1$,
\item\label{forb:gaps} $(\p_u, \g, \p_v),\; (\q_u,\g, \q_v)$, where $(u, v)\in\E$ or $u=v$,
\item\label{forb:pq} $(\p_u, \alpha, \q_v)$, for any atom $\alpha$ and any $u, v\in V$,
\item\label{forb:no triangle} $(c_{uv}, c_{vw}, c_{uw})$, any $c\in\set{\r, \sfb}$ and any $u, v, w\in V$.
\end{enumerate}
Note that the set of atoms does not include a subset $S_{\infty\infty}$, so  triples \ref{forb:S} and \ref{forb:y} forbid all triples of non-identity atoms of the form $(\alpha, \beta, \gamma)$ where $\alpha\in S_{\infty G}$ and $\beta\in S_{G \infty}$, thus $\p_u;\cv\p_v$ equals $1'$ if $u=v$ else it is zero, for all $u, v\in V$.  Also, there are no atoms $c_{uu}$ for $c\in\set{\r, \sfb}$ since $(u, u)\not\in\E$, so by forbidden triples \ref{forb:y},  \ref{forb:S} \ref{forb:indices2}, \ref{forb:gaps} we have $\cv\p_u;\p_u=1'$.  Similarly $\q_u;\cv\q_v\leq 1'$ and $\cv\q_u;\q_u=1'$.

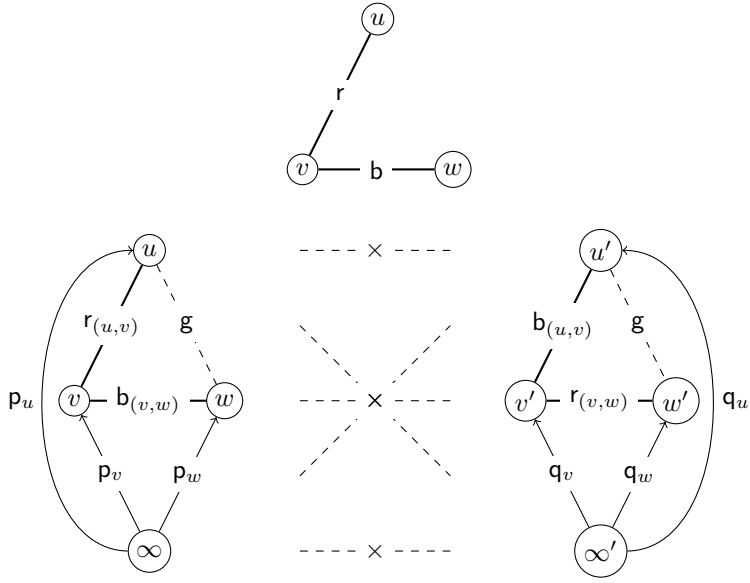
\begin{figure}\begin{tikzpicture}
\node[vertex](u) at (1,2) {$\;u\;$};
\node[vertex](v) at (0,0) {$\;v\;$};
\node[vertex](w) at (2,0) {$\;w\;$};
\draw [thick] (u) -- node[fill=white] {$\r$} (v) --node[fill=white]{$\sfb$} (w);
\end{tikzpicture}\\ \vspace{.2in}
\begin{tikzpicture}
\node[vertex](u) at (1,2) {$\;u\;$};
\node[vertex](v) at (0,0) {$\;v\;$};
\node[vertex](w) at (2,0) {$\;w\;$};
\node[vertex](i) at (1, -2){$\;\infty\;$};
\draw [thick] (u) -- node[fill=white] {$\r_{(u,v)}$} (v) --node[fill=white]{$\sfb_{(v,w)}$} (w);
\node[vertex](u') at (7,2) {$\;u'\;$};
\node[vertex](v') at (6,0) {$\;v'\;$};
\node[vertex](w') at (8,0) {$\;w'\;$};
\node[vertex](i') at (7, -2){$\;\infty'\;$};
\draw [thick] (u') --node[fill=white] {$\sfb_{(u,v)}$} (v') --node[fill=white]{$\r_{(v,w)}$} (w');
\draw[->](i)--node[fill=white]{$\p_v$}(v);
\draw[->](i)--node[fill=white]{$\p_w$}(w);
\draw[dashed] (u) --node[fill=white]{$\g$}(w);
\draw[dashed] (u') --node[fill=white]{$\g$}(w');
\draw[->](i')--node[fill=white]{$\q_v$}(v');
\draw[->](i')--node[fill=white]{$\q_w$}(w');

\draw[dashed](3,-1) -- (5,1);
\draw[dashed](3,1) -- (5,-1);
\draw[dashed](3,0) --node[fill=white]{$\times$} (5,0);
\node at (4,0) {$\times$};

\draw[bend angle = 90,->](i)[bend left] to (u);
\node at (-.7,0) {$\p_u$};

\draw[dashed] (3,-2)--node[fill=white]{$\times$}(5,-2);
\draw[dashed] (3,2)--node[fill=white]{$\times$}(5,2);
\draw[bend angle = 90,->](i')[bend right] to (u');
\node at (8.8,0) {$\q_u$};
\end{tikzpicture}\\
\caption{\label{fig:feeble}A 2-colouring of the edges $(u, v), (v, w)$ of a graph with nodes $u, v, w$ (above) and a consistent atomic network witnessing all atoms in $(S, E, \cv{}, C)$, obtained from it (below).  The atomic network determines a feeble representation.}
\end{figure}

The reduction maps $(V, \mathscr{E})$ to the complex algebra of the atom
structure $(S, E, \cv{ }, C)$ just defined.
Before we prove the correctness of this reduction, we mention the intended roles of the atoms.
Given any consistent, atomic network witnessing all atoms, the atom $\x$ will define a bipartite edge relation over the set of nodes of the network (as in Examples~\ref{examples}.\ref{3-ex}).  Each part of the network contains a copy of the graph together with a single auxiliary node and will encode a valid colouring of the graph, with atoms $\sfr_{uv}$ and $\sfb_{uv}$ 
 encoding edges coloured red and blue, respectively, while $\sfg$ (``gap'') corresponds to
non-edges.    The atom $\p_u$ labels the edge from the auxiliary node to $u$ in one part of the network while the atom $\q_u$ is used in the other part.  They are used to
ensure that the encoding describes the way edges fit together correctly.   Conversely, given an arbitrary edge colouring of $G$ there is dual colouring obtained by swapping red and blue edge labels.  By using the colouring on one copy of $G$ and the dual colouring on another copy of $G$ we may construct a consistent, atomic network in which each atom $\r_{uv}$ and $\sfb_{uv}$ is witnessed.  Such an atomic network is illustrated in Figure~\ref{fig:feeble}\nb{R3}, based on a two colouring of a three node graph.

Now we make this more precise by checking that the reduction is correct.  Suppose $\rho:\E\rightarrow\set{\r, \sfb}$ is a symmetric colouring avoiding monochrome triangles.
Let $V_+$ consist of the nodes of $V$ together with a single extra point $\infty$ and let $V_+'=\set{x':x\in V_+}$ be a set of the same size as $V_+$ disjoint from it.
 We will define an atomic network $(V_+\cup V_+' , \; \lambda)$ avoiding all forbidden triples (so consistent) and witnessing every atom, thereby showing that the complex algebra over $(S, E, \cv{}, C)$ has a feeble  representation.  To define $\lambda$:
\begin{align*}
\lambda(x, x) &= 1'&& x\in V_+\cup V_+'   \\
\lambda (x, y')&=\x&&x\in V_+,\; y'\in V_+'   \\
\lambda(\infty, u)&=\p_u&&u\in V\\
\lambda(\infty', u')&=\q_u&&u\in V\\
\lambda(u, v)&=c_{uv}&&(u, v)\in\E,\; \rho(u, v)=c\;\in\set{\r,\sfb}\\
\lambda(u', v')&=\bar c_{uv}&&(u, v)\in\E,\; \rho(u, v)=c\\
\lambda(u, v)=\lambda(u', v')&=\g&&(u, v)\not\in\E,\;u\neq v
\end{align*}
where $\bar c$ denotes `the other colour', for $c\in\set{\r, \sfb}$, and each converse edge is labelled by the converse atom, e.g. $\lambda(u, \infty)=\cv\p_u$.
It is a routine check that this defines a consistent atomic network and that
every atom labels at least one edge.  By Lemma~\ref{lem:feeble}, $(S, E, \cv{ }, C)$ has a feeble representation.

For the converse, suppose the complex algebra over $(S, E, \cv{}, C)$ has a feeble representation, by Lemma~\ref{lem:feeble} there is a consistent,  atomic network  $(N, \lambda)$ witnessing all atoms.  We must show
that $G$ is a yes instance of the monochromatic triangle problem.   

By forbidden triple \ref{forb:y} we have that for any $x, y, z\in N$ either none or exactly two of $\set{\lambda(x, y), \lambda(x, z), \lambda(y, z)}$ equal $\x$.  It follows that  the set of pairs $\set{(x, y):x, y\in N,\;\lambda(x, y)=\x}$ forms a complete bipartite graph with nodes $N$.  So $N$ is the disjoint union $N_1\cup N_2$, say, and $\lambda(x, y)=\x$ iff either $x\in N_1,\; y\in N_2$ or $x\in N_2,\; y\in N_1$. 

Since all atoms are witnessed there are nodes $z, x\in N$ such that $\lambda(z, x)=\p_u$, without loss $z, x\in N_1$.   By forbidden triple \ref{forb:S}, for every $y\in N_1\setminus\set z$ the label $\lambda(z, y)$ belongs to $S_{\infty G}$, so it is $\p_v$ or $\q_v$ for some $v\in V$, but it cannot be $\q_v$ by forbidden triple \ref{forb:pq}.  By forbidden triples \ref{forb:indices2} and \ref{forb:gaps},
\[ \lambda(x, y) = \left\{\begin{array}{ll} 1'&\mbox{if }u=v\\  c_{uv}&\mbox{if $(u, v)\in\E$,  for some }c\in\set{\r, \sfb}\\
\g&\mbox{otherwise}
\end{array}\right.
\]
Hence the map $^*:N_1\setminus\set z\rightarrow V$ which maps $y\in N_1\setminus\set z$ to $v\in V$ iff $\lambda(z, y)=\p_v$ is a well-defined injection.  Since each atom $\p_v$ is witnessed in the network $*$ is surjective, hence a bijection from $N_1$ onto $V$.  Define a function $\rho:\E\rightarrow\set{\r, \sfb}$ by letting $\rho(u^*, v^*)=c$ if $\lambda(u, v)=c_{uv}$, for $c\in\set{\r, \sfb}$.  Since $\cv{\r_{uv}}=\r_{vu}$ and $\cv{\sfb_{uv}}=\sfb_{vu}$ this colouring function is symmetric.  By forbidden triple \ref{forb:no triangle}, $\rho$ is a valid colouring of the graph.
\end{proof}

\section{Axiomatisability}\label{sec:NFA}

\begin{theorem}\label{thm:nfa}
If $\c K$ is a class of  algebras containing all strongly representable relation algebras and contained in the class of all feebly representable non-associative algebras, then $\c K$ has no finite axiomatisation in first order logic.\end{theorem}
\begin{proof}
An $n$-colouring of a set $S$ is a symmetric function $\rho$ mapping pairs of distinct elements of $S$ to a set of $n$ colours, avoiding monochromatic triangles.
Let $n\geq 3$, let $k(n)$ be the smallest integer such that there is no  $n$-colouring on a set with $k(n)$ elements (the Ramsey number) and let $\alpha_n$ be the atom structure with atoms
\[ 
\left\{1'\right\}\cup\left\{a_0^k\colon k<\frac{k(n)(k(n)-1)}2\right\}\cup\{a_i:0<i<n\}
\]
All atoms are self-converse.  All triples of atoms are consistent except those
of the form $(1', x, y)$ for $x\neq y$ (and Peircean transforms of these) and
triples of atoms with the same subscript, i.e. $(a_i, a_i, a_i)$ for $0<i<n$ or
$(a_0^k, a_0^{k'}, a_0^{k^*})$, for $k, k', k^*<\frac{k(n)(k(n)-1)}2$.   Let $\c A_n$ be the complex algebra of $\alpha_n$ (sometimes called a \emph{Monk Algebra}, see
\cite[Definition~15.2]{HH:book}).  Then $\c A_n$ has no feeble
representation because any atomic network witnessing all 
atoms has at least $\frac{k(n)(k(n)-1)}2$ distinct edges (one for each atom
$a_0^k$) hence $k(n)$ distinct points,  but no such network can be consistent
since there is no $n$-colouring of a set with this many elements. 

Now consider a non-principal ultraproduct $\alpha=\prod_{n\in\omega}\alpha_n/U$ of the
$\alpha_n$, where $U$ is a non-principal ultrafilter over $\omega$. The atoms of
$\alpha$ are (up to isomorphism) 
$\{1'\}\cup\{a_0^k\colon k<\kappa\}\cup\{a_i:0<i<\eta\}$ 
where $\kappa, \eta$ are infinite ordinals, $\kappa$ is the non-principal
ultraproduct of the $k(n)$s and $\eta$ is the non-principal ultraproduct of the
$n$s.  All atoms are self converse and the consistent triples of atoms are as in
the definition of $\alpha_n$.  We claim that the complex algebra $\Cm(\alpha)$
is a strongly representable relation algebra.  Since $\RRA$ is a variety, it
suffices to prove that every finitely generated subalgebra of $\Cm(\alpha)$ is
representable. For this, observe that an element $a\geq a_i+a_j$, for any distinct $i, j>0$ is \emph{flexible} in that $a;x \geq 1-1'$ for any atom $x\neq 1'$.  So a finite subset $S$ of the elements of $\c A$ generates a finite boolean subalgebra of the boolean part of $\c A$,  this finite boolean subalgebra generates (using $1', \; \cv{}$ and $;$) a finite subalgebra of $\c A$ and this finite subalgebra contains at least one flexible atom.  It is known that every relation algebra with a
flexible atom is representable \cite{Comer:comb84} or \cite[Theorem~5.19]{Ma82},
this proves the claim.   
(For a proof in a more complicated but similar case, see
\cite[Theorem~24, Proposition~26, Theorem~27]{HH:raca2}.)  Observe that the ultraproduct $\Pi_{n\in\omega}\c A_n/U$ is a subalgebra of the complex algebra of $\alpha$, so it is also strongly representable. Thus,  none of the algebras $\c A_n$ has a feeble representation so $\c A_n\not\in\c K$ (for $n<\omega$) but an ultraproduct of them has a strong representation and is in $\c K$.
 By \Los' Theorem, $\c K$ cannot be defined by finitely many axioms.
\end{proof}
By this theorem, the following classes  cannot be defined by finitely many axioms: the class of non-associative algebras with feeble representations; with qualitative representations; with strong representations.

\section{Conclusion}
Let us review some of the advantages of qualitative representations.
For many applications in knowledge representation, it is natural to express that a certain relation \emph{may} be decomposed in a certain way without insisting that such a decomposition must always exist, for example if one asserts that $y$ occurs strictly later than $x$, it might be the case that there is a $z$ occurring  in between, but not necessarily.  In order to model real world  applications using relation-like algebras, it is often necessary to consider a type of representation  for our algebras that would be ruled out if we were to restrict ourselves to strong representations.  
 For relation algebras such as RCC8 where we may wish to consider disconnected regions and regions with holes it has been known for some time that  strong relation algebra representations are problematic.   So the fact that qualitative representations include along with strong representations a very wide class of different representations, significantly extends our ability to model various situations.

Moreover, if a finite algebra has a qualitative representation then it has a qualitative representation on a finite base.  This means, in general, that the network satisfaction problem is much easier for qualitative representations than the corresponding problem for strong representations, although for certain well-known relation algebras the two versions of the problem turn out to be equivalent.  We have seen that several representation problems become computationally much easier in the context of qualitative representations compared to the corresponding strong representation problem.

In a subsequent paper we intend to investigate the universal algebra of $\QRA$ in more depth, by  providing a recursive (but necessariliy infinite) axiomatisation of it, by considering \emph{complete} qualitative representations and by showing that $\QRA$ generates a \emph{discriminator} variety.


\section*{Acknowledgement} 
The authors wish to thank Dr.~Todd Niven for his essential input in the early development of this work.
\bibliographystyle{alpha}
\newcommand{\etalchar}[1]{$^{#1}$}

\end{document}